\documentclass{ecai}
\usepackage{graphicx}
\usepackage{latexsym}
\usepackage{amsmath,amsfonts,amssymb,amsthm}
\usepackage{xcolor}
\usepackage{lipsum}
\usepackage{titlesec}

\newtheorem{thm}{Theorem}
\newtheorem{lem}{Lemma}

\newtheorem{rem}{Remark}

\DeclareMathOperator*{\E}{\mathbb{E}}
\DeclareMathOperator*{\e}{\mathrm{err}}
\DeclareMathOperator*{\m}{\mathrm{mem}}
\DeclareMathOperator*{\Cov}{\mathrm{Cov}}
\DeclareMathOperator*{\tr}{\mathrm{tr}}
\DeclareMathOperator*{\Var}{\mathrm{Var}}

\usepackage{tikz}
\newcommand*\circled[1]{\tikz[baseline=(char.base)]{
            \node[shape=circle,draw,inner sep=2pt] (char) {#1};}}

\begin{document}

\begin{frontmatter}

\title{Long-Tail Theory under  Gaussian Mixtures}

%\author{Author(s)}
\author[A]{\fnms{Arman}~\snm{Bolatov}}
\author[B]{\fnms{Maxat}~\snm{Tezekbayev}}
\author[C]{\fnms{Igor}~\snm{Melnykov}}
\author[B]{\fnms{Artur}~\snm{Pak}}
\author[D]{\fnms{Vassilina}~\snm{Nikoulina}}
\author[E]{\fnms{Zhenisbek}~\snm{Assylbekov}\orcid{0000-0003-0095-9409}\thanks{Corresponding Author. Email: zhenisbek.assylbekov@gmail.com. Work done while at Nazarbayev University.}}

\address[A]{Department of Computer Science, School of Engineering and Digital Sciences, Nazarbayev University,\\Kabanbay Batyr 53, Astana, Kazakhstan, 010000}
\address[B]{Department of Mathematics, School of Sciences and Humanities, Nazarbayev University,\\Kabanbay Batyr 53, Astana, Kazakhstan, 010000}
\address[C]{Department of Mathematics and Statistics, University of Minnesota Duluth, Duluth, MN, USA}
\address[D]{NAVER LABS Europe, 6-8 chemin de Maupertuis, 38240 Meylan, France}
\address[E]{Department of Mathematical Sciences, Purdue University Fort Wayne, Fort Wayne, IN, USA}

\begin{abstract}
We suggest a simple Gaussian mixture model for data generation that complies with Feldman's long tail theory (2020). We demonstrate that a linear classifier cannot decrease the generalization error below a certain level in the proposed model, whereas a nonlinear classifier with a memorization capacity can. This confirms that for long-tailed distributions, rare training examples must be considered for optimal generalization to new data. Finally, we show that the performance gap between linear and nonlinear models can be lessened as the tail becomes shorter in the subpopulation frequency distribution, as confirmed by experiments on synthetic and real data.
\end{abstract}

\end{frontmatter}

\section{Introduction}
In  classical learning theory \cite{valiant1984theory,vapnik1974theory}, generalizing ability and model complexity are usually opposed to each other: the more complex the model,\footnote{By \emph{complexity} of a model we mean its ability to fit an arbitrary dataset.} %, the better it can fit an arbitrary dataset,
%and, accordingly
 the worse its generalizing ability on new data. This is well illustrated by typical curves of test and training errors as functions of the complexity of the model being trained. %(Figure~\ref{fig:err_complexity}).
\iffalse
\begin{figure}[htbp]
    \centering
    \includegraphics[width=.45\textwidth]{plots/err_complexity.png}
    \caption{Test and training errors as functions of model complexity. Source: \citet{hastie2009elements}.}
    \label{fig:err_complexity}
\end{figure}
\fi
The training error tends to decrease whenever we increase the model complexity, that is, when we try harder to fit the data. With too much fitting, the model adapts itself too closely to the training data, and will not generalize well (i.e., have large test error). 

However, modern machine learning models such as deep neural networks (DNNs) break this principle: they are usually complex enough to be able to memorize the entire training set, and nevertheless show excellent generalization ability. This phenomenon, called \textbf{benign overfitting}, was discovered empirically by Zhang~et~al.~\cite{DBLP:conf/iclr/ZhangBHRV17} and has since attracted the attention of many minds in the field of machine learning, both experimentalists and theorists. We refer the reader to the survey of Bartlett~et~al.~\cite{bartlett2021deep} and Belkin~\cite{belkin2021fit} for a more comprehensive overview of benign overfitting. 

In our opinion, the most adequate explanation for the \emph{necessity} of overfitting is the \textbf{long tail theory} of Feldman~\cite{DBLP:conf/stoc/Feldman20}, which considers learning from natural data (such as texts or images). The fact is that the distribution of such data usually consists of subpopulations, and the frequencies of subpopulations have a so-called long tail, i.e. examples from rare/atypical subpopulations will regularly occur  in both training and test samples.

As an example, consider a typical dataset consisting of movie reviews, such as SST-2 \cite{socher-etal-2013-recursive}. In this dataset, each review (more precisely, each sentence) is labeled as positive or negative. If we look at a typical positive sentence,
\begin{center}
\textit{The large-format film is \underline{well suited} to capture these musicians in \underline{full regalia} and the \underline{incredible IMAX sound system} lets you feel the beat down to your toes.}
\end{center}
we will notice that it contains positive phrases (underlined). At the same time, a typical negative sentence, for example
\begin{center}
\textit{The images \underline{lack contrast}, are \underline{murky}, and are frequently \underline{too dark} to be decipherable.}
\end{center}
includes mostly negative phrases. However, the richness of human language allows one to write negative review sentences, which nevertheless abound in positive phrases:
\begin{center}
\textit{Starts out with \underline{tremendous} promise, introducing an \underline{intriguing} and \underline{alluring} premise, only to become a \underline{monumental achievement} in practically every facet of inept filmmaking.}
\end{center}
These kinds of negative reviews are not typical, and according to Feldman's long-tail theory, they constitute a separate subpopulation (or several subpopulations) in the class of negative reviews.

A similar situation is observed in the image domain. Consider the popular MNIST dataset \cite{DBLP:journals/spm/Deng12}, which consists of handwritten digits from 0 to 9. Typically, this dataset is used for 10-class classification, where each digit is a class. If we take one of the classes, say 3, then most of the examples in this dataset look like in Figure~\ref{fig:mnist_typ}.
\begin{figure}[htbp]
    \begin{minipage}{.23\textwidth}
    \includegraphics[width=\textwidth]{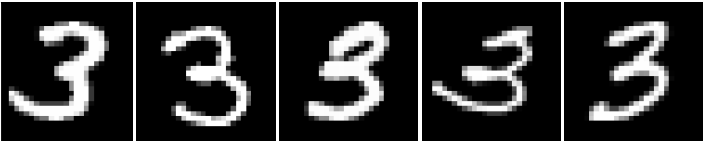}
    \caption{Typical examples of writing 3 in MNIST dataset. Source: \cite{DBLP:conf/nips/FeldmanZ20}.}
    \label{fig:mnist_typ}
    \end{minipage}\hfill\begin{minipage}{.23\textwidth}
        \includegraphics[width=\textwidth]{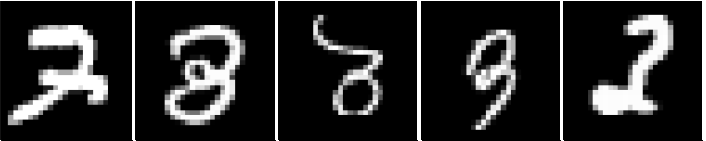}
    \caption{Atypical examples of writing 3 in MNIST dataset. Source: \cite{DBLP:conf/nips/FeldmanZ20}.}
    \label{fig:mnist_atyp}
    \end{minipage}
\end{figure}
However, there are rare and atypical examples of writing digit 3, such as in Figure~\ref{fig:mnist_atyp}, which are easily confused with other digits (for example, 7). Again, according to the long-tail theory, such rare examples should be allocated to a separate subpopulation (or several subopulations) within the class 3. 

Feldman \cite{DBLP:conf/stoc/Feldman20} showed that if the distribution over subpopulations has a long-tail (as in the examples above), then to achieve optimal performance, the learning algorithm \emph{needs} to memorize rare/atypical examples from the training set. This is formalized via a lower
bound on the generalization error, which is proportional to the number of mislabeled training examples (and the proportionality coefficient depends on the long-tailed distribution over subpopulations). Thus, in order for the learning algorithm to be able to reduce the generalization error to a minimum, it \emph{needs} to fit \emph{all} examples from the training set, including rare/atypical ones. And this, in turn, entails the need to use more complex models (with a larger number of parameters), since simple and underparameterized models are not able to memorize such atypical cases. However, Feldman's work does not provide conditions that guarantee successful learning from natural data, i.e. there are no upper bounds on the generalization error.

At the same time, it should be noted that recently we have seen an increase in the number of works in which guarantees of successful learning for \textbf{interpolating methods} are mathematically proved. For example, Chatterji~and~Long~\cite{DBLP:journals/jmlr/ChatterjiL21} showed that an overparameterized max-margin linear classifier trained on a linearly separable-with-noise data can perfectly fit the training sample (interpolate), yet generalize to new data nearly optimally.
%\begin{theorem}[\citet{DBLP:journals/jmlr/ChatterjiL21}] Assume $\mathcal{D}$ is the (unknown) true distribution from which the training sample $S$ of size $n$ is generated, $\pm\boldsymbol\mu\in\mathbb{R}^p$ are mean vectors for positive and negative classes with $s$ nonzero components, $p$ is the dimensionality of the input, i.e. $\mathbf{x}\in\mathbb{R}^p$, and $n<p<s^2$. Then if $S$ is linearly separable with noise level $\eta$, training a linear classifier on it produces the max-margin classifier $\mathbf{w}$ such that
%\begin{enumerate}
%    \item It perfectly fits the training sample: 
%    $
%    \Pr_{(x,y)\sim S}[\sign(\mathbf{w}\cdot \mathbf{x})\ne y]=0
%    $
%    \item It generalizes nearly optimally: 
%    \begin{multline*}\textstyle
%    \Pr_{(x,y)\sim\mathcal{D}}[\sign (\mathbf{w}\cdot\mathbf{x})\ne y]\\\le\eta+\exp(-c\|\boldsymbol\mu\|^4/p),
%    \end{multline*}
%    where $c$ is a positive constant.
%\end{enumerate}
%\end{theorem}
A similar result was shown by Shamir~\cite{DBLP:conf/colt/Shamir22}, and extensions to neural networks with one hidden dense layer and one hidden convolutional layer  were recently given by Frei~et~al.~\cite{DBLP:conf/colt/FreiCB22} and Cao~et~al.~\cite{cao2022benign} respectively. We are mainly concerned with the assumptions on data generation made in these works: the setup of a \emph{single}, albeit noisy,  subpopulation within each class is completely different from what Feldman~\cite{DBLP:conf/stoc/Feldman20} suggested in his long-tail theory. Moreover, in such a setup, there are non-interpolating algorithms %(Figure~\ref{fig:non_interp}) 
%\begin{figure}[htbp]
%    \centering
%    \includegraphics[width=.4\textwidth]{Figures/not_seperable_01.png}
%    \caption{Non-interpolating classifier with optimal error.}
%    \label{fig:non_interp}
%\end{figure}
with the same (or better) generalization guarantees. Accordingly, memorizing rare noisy  examples is \emph{not} necessary to achieve optimal generalization error.

In this paper, we propose a simple Gaussian mixture model for data generation  that is consistent with Feldman's long-tail theory. Further, we show that, within the framework of the proposed model, a linear classifier cannot reduce the generalization error below a certain limit, regardless of the number of parameters used. At the same time, there is a nonlinear model with a larger number of parameters, which can reduce the generalization error below this limit. Thus we show that fitting rare/atypical training examples is \emph{necessary} for optimal generalization to new data. Finally, we prove that the performance gap between linear and non-linear models can be decreased as the tail shortens in the subpopulation frequency distribution. This result is confirmed by experiments on both synthetic and real data.

\section{Data Generating Model}
\paragraph{Motivating Example.} To motivate our choice of data-generating model, let us go back to the movie review examples. For simplicity, let us imagine that we can identify positive and negative phrases in the reviews. Next, 
let us represent each review sentence as a single number 
$$
x = (\#\text{positive phrases})-(\#\text{negative phrases})
$$
It is intuitively clear that for most positive sentences, $x > 0$, and for most negative sentences, $x < 0$. However, as we mentioned in the Introduction, there are rare examples of negative review sentences that abound in positive phrases, i.e. for which $x > 0$.\footnote{And vice versa: there are rare examples of positive review sentences that abound in negative phrases. However, for ease of analysis, we will omit this case.} This observation leads us to the following data distribution model: for all positive reviews, $x$ is concentrated at the point $\mu_+ >0$; for most negative reviews, $x$ is concentrated at the point $\mu_{-}^{\text{maj}}<0$; while there is a minority of negative reviews for which $x$ is concentrated at the point $\mu_{-}^\text{min}>0$ (Figure~\ref{fig:gmm_1d}).
\begin{figure}[htbp]
    \begin{minipage}{.23\textwidth}
    \includegraphics[width=\textwidth]{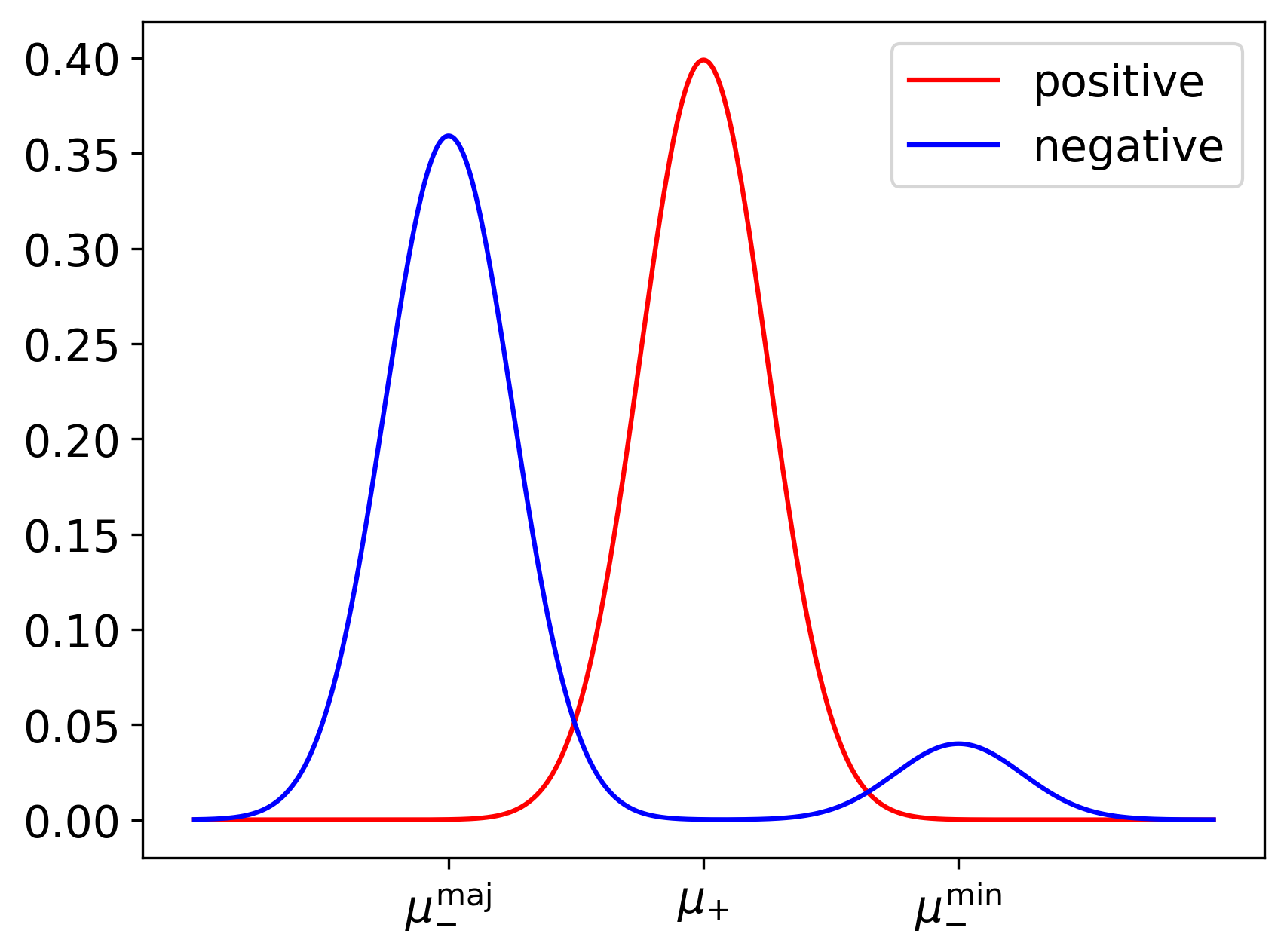}
    \caption{Simplified data distribution model.}
    \label{fig:gmm_1d}
    \end{minipage}\hfill\begin{minipage}{.23\textwidth}
    \includegraphics[width=\textwidth]{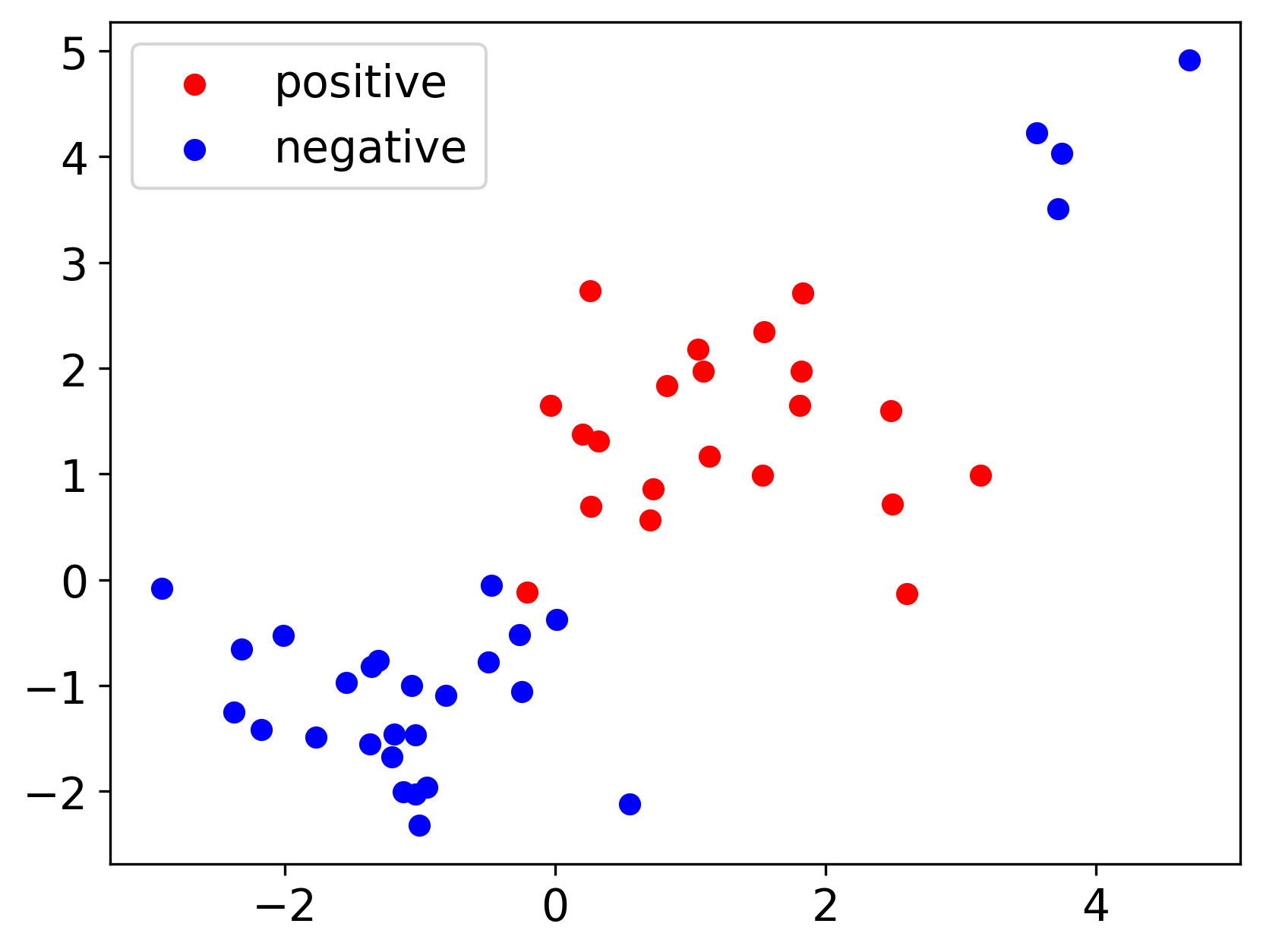}
    \caption{A sample from our data generating model.}
    \label{fig:gmm_example}
    \end{minipage}
\end{figure}
In what follows, we formalize this model.

\paragraph{Notation.} We let $\mathbb{R}$ denote the  real numbers. Bold-faced lowercase letters ($\mathbf{x}$) denote vectors in $\mathbb{R}^d$, bold-faced uppercase letters ($\mathbf{A}$, $\mathbf{X}$) denote matrices and random vectors, regular lowercase letters ($x$) denote scalars, regular uppercase letters ($X$) denote random variables. $\|\cdot\|$ denotes the Euclidean norm: $\|\mathbf{x}\|:=\sqrt{\mathbf{x}^\top\mathbf{x}}$. $\mathcal{N}(\boldsymbol\mu,\boldsymbol\Sigma)$ denotes multivariate Gaussian with mean vector $\boldsymbol\mu\in\mathbb{R}^d$ and covariance matrix $\boldsymbol\Sigma\in\mathbb{R}^{d\times d}$. `p.d.f.' stands for `probability density function', and `c.d.f.' stands for `cumulative distribution function'. %The p.d.f. and c.d.f. of ${Z}\sim\mathcal{N}(\mathbf{0},\mathbf{I})$ are denoted as $\phi(\mathbf{z})$ and $\Phi(\mathbf{z})$ respectively. 
The p.d.f. of $\mathbf{X}\sim\mathcal{N}(\boldsymbol\mu,\boldsymbol\Sigma)$ is denoted by $f(\mathbf{x};\boldsymbol\mu,\boldsymbol\Sigma)$. The p.d.f. and c.d.f. of $Z\sim\mathcal{N}(0,1)$ are denoted by $\phi(z)$ and $\Phi(z)$ respectively. We also use the standard big O notation, such as $O(\cdot)$, $\widetilde{O}(\cdot)$, $\Omega(\cdot)$, $\Theta(\cdot)$,   \cite[Chapter 3]{DBLP:books/daglib/0023376}. %In particular, given $f:\,\mathbb{R}\to\mathbb{R}$ and $g:\,\mathbb{R}\to\mathbb{R}_+$, we write $f=O(g)$ if there exist $x_0,\alpha\in\mathbb{R}_+$ such that for all $x>x_0$ we have $|f(x)|\le\alpha g(x)$. When $f:\mathbb{R}\to\mathbb{R}_+$, we write $f=\Omega(g)$ if $g=O(f)$. The notation $f=\Theta(g)$ means that $f=O(g)$ and $f=\Omega(g)$. Finally, the notation $f=\widetilde{O}(g)$ means that there exists $k\in\mathbb{N}$ such that $f(x)=O\Bigl(g(x)\ln^k(g(x))\Bigr)$.

\paragraph{The Model.} Let $\mathbf{X}\in\mathbb{R}^d$ be the feature vector, and $Y\in\{-1,+1\}$ its class label. We assume that $Y$ is a Rademacher random variable, i.e.
\begin{equation}
\Pr[Y=-1]=\Pr[Y=+1]=\frac12.\label{eq:prior}
\end{equation}
For the positive class, we assume that the class-conditional distribution of $\mathbf{X}$ is a spherical (a.k.a. isotropic) Gaussian centered at $\boldsymbol\mu\in\mathbb{R}^d$, i.e.
\begin{equation}
(\mathbf{X}\mid Y=+1)\sim\mathcal{N}(\boldsymbol\mu,\sigma^2\mathbf{I}).\label{eq:pos_class}
\end{equation}
Whereas for the negative class, the class-conditional distribution of $\mathbf{X}$ is a \emph{mixture} of two spherical Gaussians centered at $-\boldsymbol\mu$ and $3\boldsymbol\mu$:%, with mixture weights $p$ and $(1-p)$, i.e.
\begin{align}
&(\mathbf{X}\mid Y=-1,K=1)\sim\mathcal{N}(-\boldsymbol\mu,\sigma^2\mathbf{I}),\label{eq:neg_class1}\\
&(\mathbf{X}\mid Y=-1,K=2)\sim\mathcal{N}(3\boldsymbol\mu,\sigma^2\mathbf{I}).\label{eq:neg_class2}
\end{align}
Here, the latent random variable $K$ represents the mixture component. With $K=1$, features are generated from the distribution of typical negative examples, whose proportion is $p > 1/2$ of all negative examples (i.e., this is the cluster of a majority of negative examples). With $K=2$, features are generated from the distribution of atypical/rare negative examples, whose proportion is $(1-p)<1/2$ of all negative examples:
\begin{align}
\Pr[K=1\mid Y=-1]&=p,\qquad p>\frac12\label{eq:neg_c1}\\
\Pr[K=2\mid Y=-1]&=1-p.\label{eq:neg_c2}
\end{align}

We center the atypical examples at $3\boldsymbol\mu$ so that the distance between the means of neighboring Gaussians is $2\|\boldsymbol\mu\|$, and this simplifies the analysis. The assumption that the Gaussians are
isotropic with equal covariances is also made to simplify the analysis. The centers of the Gaussians are located on the same straight line to prevent the linear separability of finite samples generated from our model. We emphasize that our goal is to build a simple data generating model that is consistent with Feldman's long-tail theory. Building a model that better agrees with real data is beyond the scope of this work and is a reasonable direction for further research. However, we believe that our model captures important features of the distribution of real data, such as the presence of rare subpopulations. We also emphasize that our model makes sense when $p>1/2$, but not too close to 1 for rare examples to occur in a finite sample. %In addition, in order for our method of estimating an unknown $\boldsymbol{\mu}$ from a finite sample to give a good estimate, we require that the sample size $n$ satisfies $(1-p)^2n=\Omega(n^\delta)$ for some $\delta>0$ (Appendix~\ref{sec:mu_mom}).

The distribution over $\mathbb{R}^d\times\{-1,+1\}$ given by \eqref{eq:prior}--\eqref{eq:neg_c2} will be denoted by $\mathcal{D}$. Figure~\ref{fig:gmm_example} shows a sample of size 50 from our data model with $d =2$ and $p=0.9$.

\section{Classifiers}\label{sec:classifiers}

In this section, we consider two classifiers---Linear discriminant analysis  and Mixture discriminant analysis---and examine their performance on data generated from our model. Let $\mathcal{P}$ be a distribution over $\mathbb{R}^d\times\{-1,+1\}$. For a classifier $h:\,\mathbb{R}^d\to\{-1,+1\}$, we define its generalization error (or misclassification error rate, or simply error) with respect to $\mathcal{P}$ as
\begin{equation}
    \e_\mathcal{P}[h]:=\Pr_{\mathbf{X},Y\sim\mathcal{P}}[h(\mathbf{X})\ne Y].
\end{equation}
When $h$ is parameterized (as is the case for the classifiers that we consider), and the parameters are estimated based on a sample $S:=\{(\mathbf{X}_i,Y_i)\}_{i=1}^n$ of i.i.d. observations from $\mathcal{P}$, we denote the resulting classifier as $h_S$ and consider its expected error $\E_{S\sim\mathcal{P}^n}[\e_\mathcal{P}[h_S]]$, where expectation is over samples $S$ of size $n$ from $\mathcal{P}$.

\subsection{Linear Discriminant Analysis (LDA)} 
\textbf{LDA} \cite{https://doi.org/10.1111/j.1469-1809.1936.tb02137.x}  is a generative classifier whose simplest version makes almost the same assumptions about data distribution as our data generating model $\mathcal{D}$. The only difference is that for the negative class, not a mixture, but one Gaussian is used, i.e. instead of the assumptions \eqref{eq:neg_class1}--\eqref{eq:neg_c2}, one has
\begin{equation}
    (\mathbf{X}\mid Y=-1)\sim\mathcal{N}(\boldsymbol\mu_{-},\sigma^2\mathbf{I}).\label{eq:neg_lda}
\end{equation}
The LDA classifier that has access to the true $\boldsymbol\mu$, $\sigma$, and $p$ can be written as
\begin{equation}
h^{\text{LDA}}(\mathbf{x})=\begin{cases}
+1\quad&\text{if }f(\mathbf{x};\boldsymbol{\mu},\sigma^2_{\text{LDA}}\mathbf{I})\ge f(\mathbf{x};\boldsymbol{\mu}_{-},\sigma^2_{\text{LDA}}\mathbf{I})\\
-1&\text{otherwise}
\end{cases},\label{eq:lda_clf}
\end{equation}
where $\boldsymbol\mu_{-}$ and $\sigma^2_\text{LDA}$ are functions of $\boldsymbol\mu$, $\sigma^2$, and $p$, that can be derived under the assumptions of the data generating model $\mathcal{D}$ (see Appendix~\ref{sec:lda_err_proof}). 
It is well known \cite[Section 4.3]{DBLP:books/lib/HastieTF09} that in this case the decision boundary of the LDA consists of a set of points equidistant from $\boldsymbol\mu$ and $\boldsymbol{\mu}_{-}$, i.e. it is a hyperplane, which is the perpendicular bisector of the line segment connecting $\boldsymbol{\mu}$ and $\boldsymbol\mu_{-}$. %(Figure~\ref{fig:lda})
%\begin{figure}[htbp]
%    \centering
%    \includegraphics[width=.4\textwidth]{Figures/isotropic_cov_db.png}
%    \caption{LDA classifier.}
%    \label{fig:lda}
%\end{figure}
It is easy to see that the data distribution used in LDA is a special case of our model when $p=1$, i.e. when the proportion of atypical examples $(1-p)$ is zero and the classes are linearly separable.\footnote{with high probability over a choice of a finite sample given sufficiently large $\|\boldsymbol\mu\|$} At the same time, in our data model, for $1-p=\Omega(1/n)$, classes cannot be linearly separated, which means that LDA will fundamentally lack the ability to fit examples from the minority subpopulation of the negative class. This is formalized in the following lemma.
    
\begin{lem} \label{lem:lda_err} Let $S\sim\mathcal{D}^n$ be a random sample from our data generating model $\mathcal{D}$ with unknown $\boldsymbol\mu$ and known $\sigma^2$ and $p$. Let $h^\text{LDA}_S$ be the LDA classifier trained on $S$ with the method of moments under the assumptions \eqref{eq:prior}, \eqref{eq:pos_class}, and \eqref{eq:neg_lda}. Then 
\begin{multline}
    \E_{S\sim\mathcal{D}^n}\left[\e_\mathcal{D}[h^\text{LDA}_S]\right]
    =\frac12\left[\Phi\left(-(2p-1)\frac{\|\boldsymbol\mu\|}{\sigma}\right)\right.\\
    \left.+p\Phi\left(-(3-2p)\frac{\|\boldsymbol\mu\|}{\sigma}\right)+(1-p)\Phi\left((2p+1)\frac{\|\boldsymbol\mu\|}{\sigma}\right)\right]\\+\widetilde{O}\left(\sqrt{\frac{d}{n}}\right).\label{eq:lda_err_lb}
\end{multline}
\end{lem}
\begin{proof} See Appendix~\ref{sec:lda_err_proof}.
\end{proof}
\begin{rem} The assumption that $\sigma^2$ and $p$ are known is unrealistic. In practice, all we have is a sample from which we need to estimate all the unknown parameters of the model. However, we believe that all the main aspects of learning from long-tailed data under Gaussian mixtures can be considered at first with stronger requirements (such as known $\sigma^2$ and $p$), and only then move on to a more realistic setting.%\footnote{E.g., in introductory statistics, the Z-test is usually introduced before the t-test to test the same hypothesis about the unknown mean of a normal distribution. However, the Z-test requires knowing the variance, while the t-test does not. Starting with the Z-test's stronger requirements allows for a comprehensive understanding of hypothesis testing before moving on to the t-test.}
\end{rem}
\begin{rem}
    It is well known that estimating the parameters of a Gaussian mixture model by the maximum likelihood method is NP-hard even in the case of diagonal covariance matrices and their equality in all Gaussians \cite{DBLP:journals/jmlr/ToshD17}. In practice, likelihood maximization is carried out by heuristic methods, such as the EM algorithm \cite{10.2307/Dempster}, which do not have theoretical guarantees of convergence to a global optimum. Therefore, we cannot rely on maximum likelihood estimates in our theoretical analysis. However, it is noteworthy that under the conditions of our generative model $\mathcal{D}$, we can estimate $\boldsymbol\mu$ by the method of moments, and the resulting estimator $\hat{\boldsymbol\mu}$ concentrates fairly well near the true $\boldsymbol\mu$ (see Lemma~\ref{lem:mom_conc} in  Appendix~\ref{sec:mu_mom}).
\end{rem}

The right-hand side (RHS) of \eqref{eq:lda_err_lb} without the last term $\widetilde{O}\left(\sqrt{d/n}\right)$ is the misclassification error of the LDA classifier given by \eqref{eq:lda_clf}, i.e. when the true parameter $\boldsymbol\mu$ is known to the classifier. In practice, it is estimated from $S$ adding a so-called estimation error in the RHS of \eqref{eq:lda_err_lb}.%, and therefore the expected error of $h^\text{LDA}_S$ is inevitably greater than that of $h^\text{LDA}$. 

For $p\in(1/2,1)$, we have $(2p-1)\in(0,1)$ and $3-2p\in(1,2)$. Using the Chernoff bound $\Phi(-x)\le\exp(-x^2/2)$ for $x>0$, and $\Phi(x)=1-\Phi(-x)$, we can qualitatively assess the bound \eqref{eq:lda_err_lb} as
\begin{multline}
\E_{S\sim\mathcal{D}^n}\left[\e_\mathcal{D}[h_S^\text{LDA}]\right]\\=\frac{1-p}{2}+\exp\left(-\Omega\left(\frac{\|\boldsymbol\mu\|^2}{2\sigma^2}\right)\right)+\widetilde{O}\left(\sqrt{\frac{d}{n}}\right).\label{eq:lda_err_simple}
\end{multline}
This confirms the impossibility of the LDA classifier to reduce the generalization error arbitrarily close to zero, no matter how far we place the Gaussians from each other. Moreover, we note that the first term in \eqref{eq:lda_err_simple} does not depend on $d$, the dimensionality of the sample space. Therefore, regardless of the dimensionality, the LDA classifier will \emph{not} be able to interpolate the training sample when $p<1$, i.e. when there is a minority subpopulation in the negative class. This is in stark contrast with the previous studies on interpolating linear methods \cite{DBLP:journals/jmlr/ChatterjiL21,DBLP:conf/colt/Shamir22}.

%\textcolor{blue}{[If time allows, compare to Feldman's lower bound.]}

\subsection{Mixture Discriminant Analysis (MDA)} 
\textbf{MDA} \cite{https://doi.org/10.1111/j.2517-6161.1996.tb02073.x} is a generative classifier that in general assumes that the data in each class is generated from a mixture of Gaussians. In our case, we can consider a version of MDA that makes precisely the assumptions \eqref{eq:prior}--\eqref{eq:neg_c2} that our data generating model $\mathcal{D}$ makes. Hence, the MDA classifier (that knows the true values of the parameters) can be written as
\begin{equation}
    h^\text{MDA}(\mathbf{x})=\begin{cases}
+1\quad&\text{if}\quad\frac12f(\mathbf{x};\boldsymbol\mu,\sigma^2\mathbf{I})\ge\frac{p}{2}f(\mathbf{x};-\boldsymbol\mu,\sigma^2\mathbf{I})\\&\text{and}\quad\frac12f(\mathbf{x};\boldsymbol\mu,\sigma^2\mathbf{I})\ge\frac{1-p}{2}f(\mathbf{x};3\boldsymbol\mu,\sigma^2\mathbf{I})\\
-1 &\text{otherwise}
\end{cases}
    \label{eq:mda_clf}
\end{equation}
Obviously, such an MDA classifier can take into account the presence of a minority subpopulation in the negative class, since it has the ability to fit a separate third Gaussian to this subpopulation. Not surprisingly, the MDA classifier \eqref{eq:mda_clf} has a near-to-optimal generalizing ability, as presented in the following lemma.
\begin{lem} \label{lem:mda_err}
    Let $S\sim\mathcal{D}^n$ be a random sample from our data generating model $\mathcal{D}$ with unknown $\boldsymbol\mu$ and known $\sigma^2$ and $p$. Let $h^\text{MDA}_S$ be the MDA classifier trained on $S$ with the method of moments under the assumptions \eqref{eq:prior}--\eqref{eq:neg_c2}. Then 
\begin{multline}
\E_{S\sim\mathcal{D}^n}\left[\e_\mathcal{D}[h^\text{MDA}_S]\right]\le\frac12\left[\Phi\left(-\frac{\|\boldsymbol\mu\|}\sigma+\frac{\sigma\ln p}{2\|\boldsymbol\mu\|}\right)\right.\\
+\Phi\left(-\frac{\|\boldsymbol\mu\|}\sigma+\frac{\sigma\ln(1-p)}{2\|\boldsymbol\mu\|}\right)+p\cdot \Phi\left(-\frac{\|\boldsymbol\mu\|}{\sigma}-\frac{\sigma\ln p}{2\|\boldsymbol\mu\|}\right)\\
\left.+(1-p)\cdot \Phi\left(-\frac{\|\boldsymbol\mu\|}\sigma-\frac{\sigma\ln(1-p)}{\|\boldsymbol\mu\|}\right)\right]+\widetilde{O}\left(\sqrt{\frac{d}{n}}\right).\label{eq:mda_err_ub}
\end{multline}
\end{lem}
\begin{proof}
    See Appendix~\ref{sec:mda_err_proof}.
\end{proof}
Arguing as in the case of LDA, we can estimate the order of the bound \eqref{eq:mda_err_ub} as
\begin{equation}
\E_{S\sim\mathcal{D}^n}\left[\e_{\mathcal{D}}[h^\text{MDA}_S]\right]\le\exp\left(-\Omega\left(\frac{\|\boldsymbol\mu\|^2}{2\sigma^2}\right)\right)+\widetilde{O}\left(\sqrt{\frac{d}{n}}\right).\label{eq:mda_err_simple}
\end{equation}
Thus, by placing the Gaussians far enough apart, the optimal error of the MDA classifier can be made arbitrarily close to zero. We emphasize that this is only due to the ability of the MDA classifier to fit (memorize) examples from the minority subpopulation $\mathcal{N}(3\boldsymbol\mu,\sigma^2\mathbf{I})$ of the negative class. Roughly speaking, MDA has the ability to allocate some of its parameters for fitting atypical examples, while LDA simply does not have such an opportunity. 

Finally, we remark that the term $\widetilde{O}\left(\sqrt{d/n}\right)$ in the RHS of \eqref{eq:mda_err_ub} is due to the error in estimating the model parameter $\boldsymbol\mu$ from the training sample $S$.

\section{Performance Gap between LDA and MDA}

Using the bounds \eqref{eq:lda_err_simple} and \eqref{eq:mda_err_simple} we can already estimate the expected difference between the LDA and MDA errors as
\begin{multline}
\E_{S\sim\mathcal{D}^n}\left[\e_\mathcal{D}[h_S^\text{LDA}]-\e_\mathcal{D}[h_S^\text{MDA}]\right]\\
\ge\frac{1-p}{2}-\exp\left(-\Omega\left(\frac{\|\boldsymbol\mu\|^2}{2\sigma^2}\right)\right)+\widetilde{O}\left(\sqrt{\frac{d}{n}}\right).\label{eq:main_simple}
\end{multline}
However, a closer analysis gives us the following
\begin{thm}\label{thm:main}
Let $S\sim\mathcal{D}^n$ be a random sample from our data generating model $\mathcal{D}$, and let $h^\text{LDA}_S$ be the LDA classifier trained on $S$ under the assumptions \eqref{eq:prior}, \eqref{eq:pos_class}, \eqref{eq:neg_lda}, and let $h^\text{MDA}_S$ be the MDA classifier trained on $S$ under the assumptions \eqref{eq:prior}--\eqref{eq:neg_c2}. Then
\begin{multline}
\E_{S\sim\mathcal{D}^n}\left[\e_\mathcal{D}[h_S^\text{LDA}]-\e_\mathcal{D}[h_S^\text{MDA}]\right]\\\ge
\frac{1-p}2 -\exp\left(-\frac{\|\boldsymbol\mu\|^2}{2\sigma^2}\right) +\widetilde{O}\left(\sqrt{\frac{d}{n}}\right).   \label{eq:main}
\end{multline}
\end{thm}
\begin{proof}
    See Appendix~\ref{sec:main_proof}.
\end{proof}
The advantage of the bound \eqref{eq:main} is that, in comparison with \eqref{eq:main_simple}, here the second term in the RHS is written explicitly, i.e., without using the big O notation. This was done through careful analysis of the original bounds from Lemmas~\ref{lem:lda_err}~and~\ref{lem:mda_err}.

Theorem 1 implies the main conclusion of our work: \emph{there is a performance gap between a simple model that is unable to memorize rare examples from the tail of the distribution, and a complex model that is able to fit such examples. Moreover, the gap can be made smaller when the proportion of atypical examples is smaller}. From \eqref{eq:lda_err_lb} and \eqref{eq:mda_err_ub}, it is easy to see that for the ``ideal'' LDA and MDA (that have access to the true $\boldsymbol\mu$), we have
\begin{align*}
\e_\mathcal{D}\left[h^\text{LDA}\right]\,\,{\stackrel{p\to1}{\longrightarrow}}\,\,\Phi\left(-\frac{\|\boldsymbol\mu\|}{\sigma}\right),\quad \e_\mathcal{D}\left[h^\text{MDA}\right]\,\,{\stackrel{p\to1}{\longrightarrow}}\,\,\Phi\left(-\frac{\|\boldsymbol\mu\|}{\sigma}\right).
\end{align*}
I.e., the gap between LDA and MDA is minimal at $p=1$ (when there is no minority subpopulation in the negative class), which is expected.

\paragraph{Implications for Real Data.} Unfortunately, our conclusion is practically impossible to test directly on real data, since we cannot be sure that its distribution resembles our model $\mathcal{D}$, and can be more complex. Moreover, our conclusion is drawn within the framework of generative classification models LDA and MDA, while in practice discriminative models such as logistic regression and multilayer neural networks are usually used, which make much fewer assumptions about the distribution of data. However, we will be able to test our conclusion \emph{indirectly} in realistic settings if there is a way to identify training examples from rare subpopulations. Fortunately, this is exactly what the memorization score introduced by Feldman and Zhang does \cite{DBLP:conf/stoc/Feldman20,DBLP:conf/nips/FeldmanZ20}.

For a learning algorithm $A$ operating on a dataset $S=\{(\mathbf{x}_i,y_i)\}_{i=1}^n$, the amount of label memorization by $A$ on example $(\mathbf{x}_i,y_i)\in S$ is defined as 
\begin{multline}
    \m[A,S,i]\\:=\Pr_{h\leftarrow A(S)}[h(\mathbf{x}_i)= y_i]-\Pr_{h\leftarrow A(S^{\setminus i})}[h(\mathbf{x}_i)= y_i],\label{eq:mem_score}
\end{multline}
where $S^{\setminus i}$ denotes the dataset $S$ with $(\mathbf{x}_i, y_i)$ removed and probability is taken over the randomness
of the algorithm $A$ (such as random initialization). One thing to keep in mind is  that the memorization score itself must be calculated through a learner that \emph{can} memorize, for example, an MDA with enough components, a neural network, or a nearest neighbor classifier. 

As we can see, the memorization score will be high for examples that are difficult (or even impossible) to correctly classify using other examples in $S$, provided that the learning algorithm is flexible enough to (almost) completely fit the training set. For example, under our data generating model, for $p$ close enough to 1 (so that $1-p = \Theta(1/n)$), these are precisely the points generated by the minority subpopulation of the negative class. Accordingly, the shortening of the tail, i.e. making $p$ closer to 1 can be simulated by discarding examples from the training set for which the memorization score is high. But the distribution of the test sample will not change, since we are not discarding top memorized examples from it. This is because the memorization score can only be calculated on the training sample and, accordingly, the most memorized examples can only be discarded from the long tail of the training sample. 

For clarity, let us re-denote the distribution given by formulas \eqref{eq:prior}--\eqref{eq:neg_c2} as $\mathcal{D}_p$. Then we are interested in the expected error of the classifier, which was trained on a sample from $\mathcal{D}_q$, but is tested on a sample from $D_p$, i.e. $\E_{S\sim\mathcal{D}_q}\left[\e_{\mathcal{D}_p}[h_S]\right]$. Fortunately, the analysis of this case resembles the analysis of the case when $q = p$. Denoting $q:=1-\frac1t$, $t>2$, we can prove the following (asymptotic in $t$) results for the LDA and MDA classifiers.
\begin{thm}\label{thm:train_test_diff}
    Let $S\sim\mathcal{D}_{1-1/t}^n$. Then for the LDA and MDA classifiers trained on $S$ but evaluated on examples from $\mathcal{D}_p$ we have
    \begin{align}
        &\E_{S\sim\mathcal{D}_{1-1/t}}\left[\e_{\mathcal{D}_p}[h^\text{LDA}_S]\right]=\frac{1+p}2\Phi\left(-\frac{\|\boldsymbol\mu\|}{\sigma}\right)\notag\\
        &\qquad+\frac{1-p}2\underbrace{\Phi\left(\frac{3\|\boldsymbol\mu\|}{\sigma}\right)}_{\circled{A}}+\Theta\left(\frac1t\right)+\widetilde{O}\left(\sqrt{\frac{d}{n}}\right)\label{eq:lda_pq_bound}\\
        &\E_{S\sim\mathcal{D}_{1-1/t}}\left[\e_{\mathcal{D}_p}[h^\text{MDA}_S]\right]\le\frac{1+p}2\Phi\left(-\frac{\|\boldsymbol\mu\|}{\sigma}\right)\notag\\
        &\qquad+\frac{1-p}2\underbrace{\Phi\left(-\frac{\|\boldsymbol\mu\|}{\sigma}+\frac{\sigma\ln t}{2\|\boldsymbol\mu\|}\right)}_{\circled{B}}+\Theta\left(\frac1t\right)+\widetilde{O}\left(\sqrt{\frac{d}{n}}\right)\label{eq:mda_pq_bound}
    \end{align}
\end{thm}
\begin{proof}
    See Appendix~\ref{sec:train_test_diff_proof}.
\end{proof}
As we can see, the difference between the bounds \eqref{eq:lda_pq_bound} and \eqref{eq:mda_pq_bound} is mainly in the terms marked as $\circled{A}$ and $\circled{B}$. It is clear that $\circled{A}>\circled{B}$ as long as $t<\exp(8\|\boldsymbol{\mu}\|^2/\sigma^2)$. For example, when $\|\boldsymbol{\mu}\|=2$, and $\sigma=1$, we have $\exp(8\|\boldsymbol{\mu}\|^2/\sigma^2)\approx8\cdot10^{13}$. Thus, the gap between LDA error and the upper bound on MDA error remains feasible even for large values of $t$ (i.e. when $q$ is close to 1). %Note that $\circled{A}$ mildly depends on $t$ (and, accordingly, on $q:=1-1/t$), which means that the expected error of the LDA classifier is less affected by the difference between the distributions $\mathcal{D}_p$ and $\mathcal{D}_q$. Such ``stability'' is explained by the fact that LDA consistently underfits in both cases: when $q=p$ and when $q\ne p$. At the same time, in the bound for the MDA, the term $\circled{B}$ depends on $t$, and the bound grows with the growth of $t$, i.e. when $q\to 1$. For example, when $t=\exp(8\|\boldsymbol\mu\|^2/\sigma^2)$, MDA error is close to the LDA error. And when $t\to\infty$, i.e. when there are no atypical examples in the training sample, the MDA error bound becomes even worse (asymptotically in $t$) than that of the LDA error.

\section{Experiments} \label{sec:experiments}
In this section, we empirically validate the predictions from our theory for synthetic data (generated from our model $\mathcal{D}$) as well as for real data, the distribution of which is not necessarily the same as $\mathcal{D}$, but shares its main characteristics, such as the presence of minority subpopulations in at least one of the classes.\footnote{The code for reproducing the experiments is at \url{https://github.com/armanbolatov/long_tail}. The random seeds we used are indicated in the code.}

\subsection{Synthetic Data}
First of all, we verify our error bounds from Lemmas \ref{lem:lda_err} and \ref{lem:mda_err} experimentally. To do this, we generate training and test sets from our data model $\mathcal{D}$, fit LDA and MDA to the training set, compute the misclassification errors on the test set, and compare with theoretical bounds \eqref{eq:lda_err_lb} and \eqref{eq:mda_err_ub} modulo the asymptotic terms $\widetilde{O}(\sqrt{d/n})$. Since the bounds depend on several parameters, we vary each of these parameters while keeping the others fixed. Unless otherwise specified, the default values of the parameters are: $d = 50$, $p = 0.9$, $\|\boldsymbol{\mu}\|= 2$, $\sigma = 1$, $n = 7000$. Test samples are of size $n_\text{test}=3000$. For each  variable parameter value, we generate 10 training and test samples and estimate the generalization errors with  95\% confidence intervals across test samples.

\paragraph{Dependence on $\|\boldsymbol\mu\|$.} To test the dependence of error bounds on $\|\boldsymbol\mu\|$, we vary $\|\boldsymbol\mu\|$ in the interval ${[2, 6]}$ with a step 0.08. For each value of $\|\boldsymbol{\mu}\|$, we take a random direction in $\mathbb{R}^d$, and place $\boldsymbol{\mu}$, $-\boldsymbol\mu$, and $3\boldsymbol\mu$ along that direction.   The results of the experiments are shown in Figure~\ref{fig:mu_norm}. 
\begin{figure*}
    \includegraphics[width=.3\textwidth]{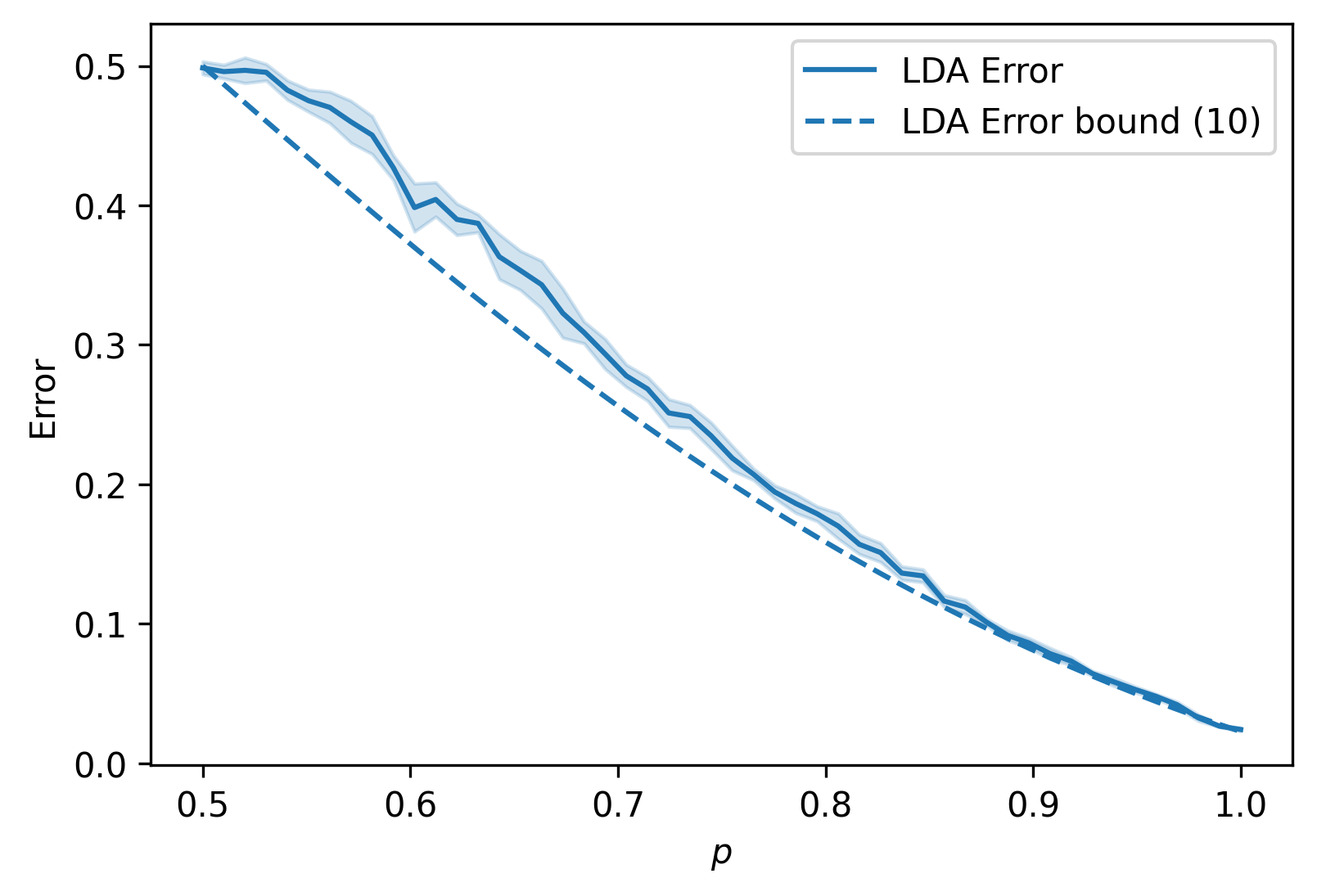}\hfill\includegraphics[width=.3\textwidth]{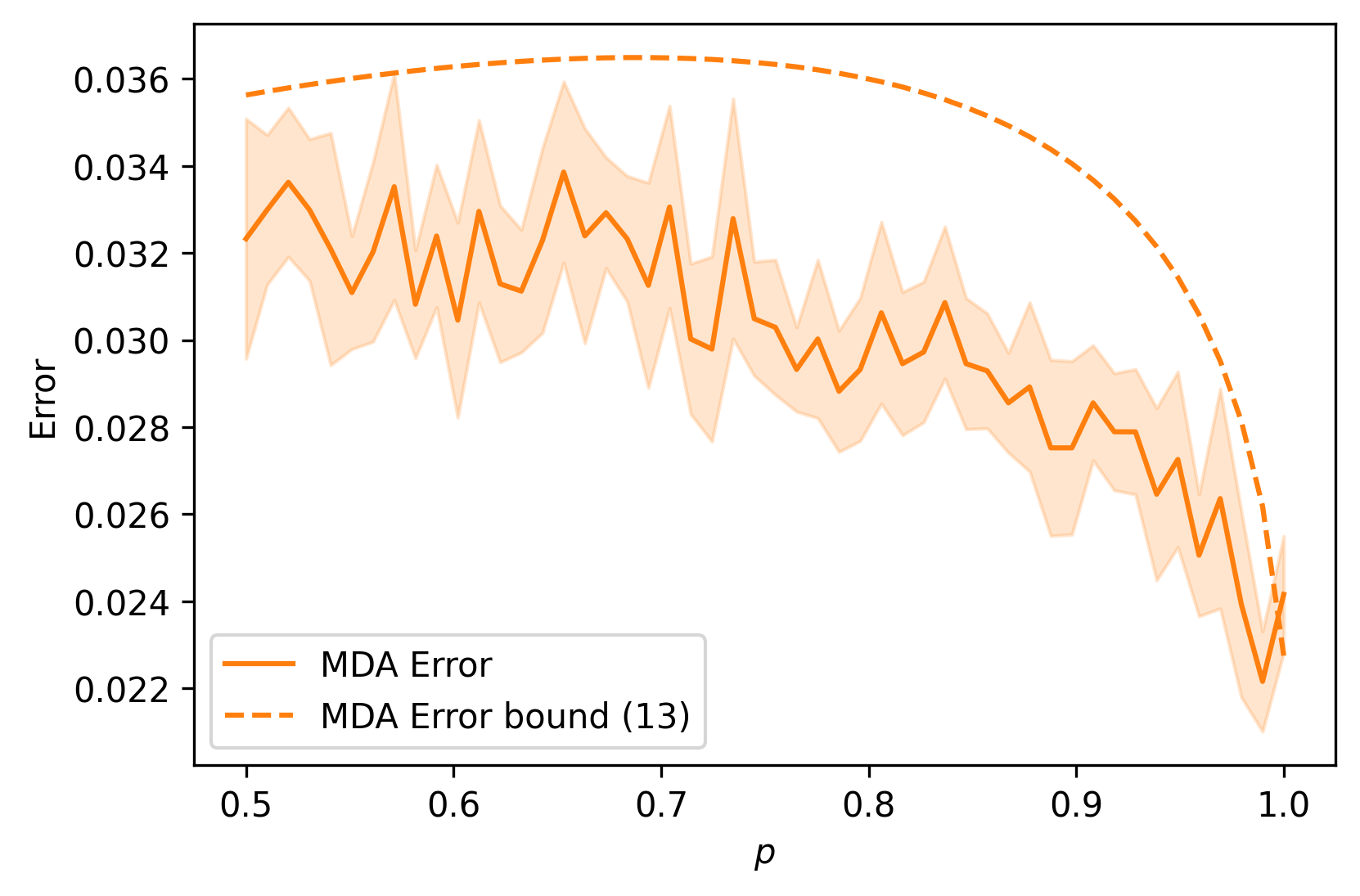}\hfill\includegraphics[width=.3\textwidth]{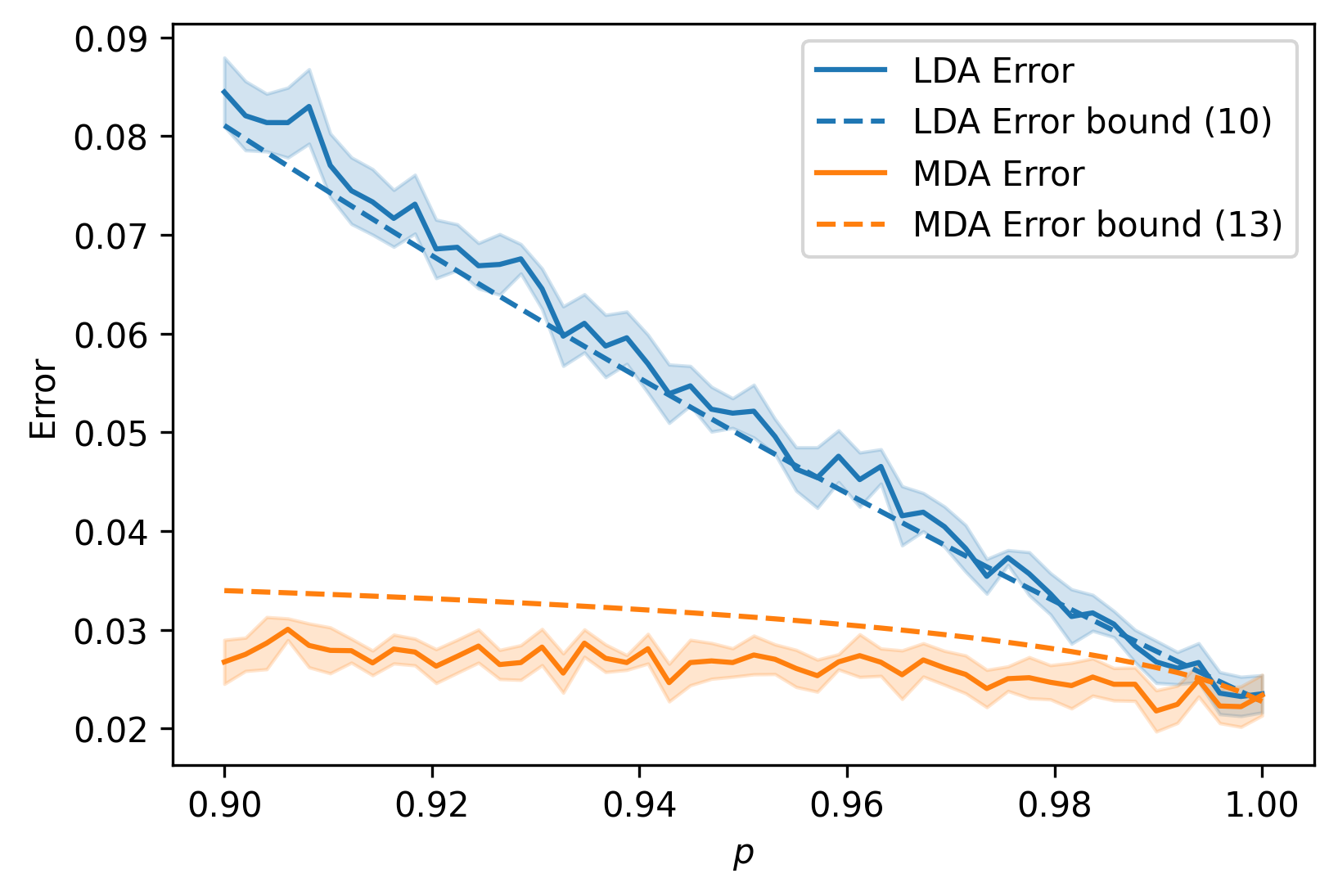}
    \caption{Comparison of empirical errors (solid) and theoretical error bounds (dashed) when $p$ varies. The shaded areas are 95\% confidence bands around the average across 10 runs. The values of the remaining parameters are fixed as follows: $d=50$,  $\|\boldsymbol{\mu}\|=2$, $\sigma=1$, $n=7000$, $n_\text{test}=3000$.}
    \label{fig:p}
\end{figure*}
\begin{figure*}
    \begin{minipage}[t]{.3\textwidth}
    \includegraphics[width=\textwidth]{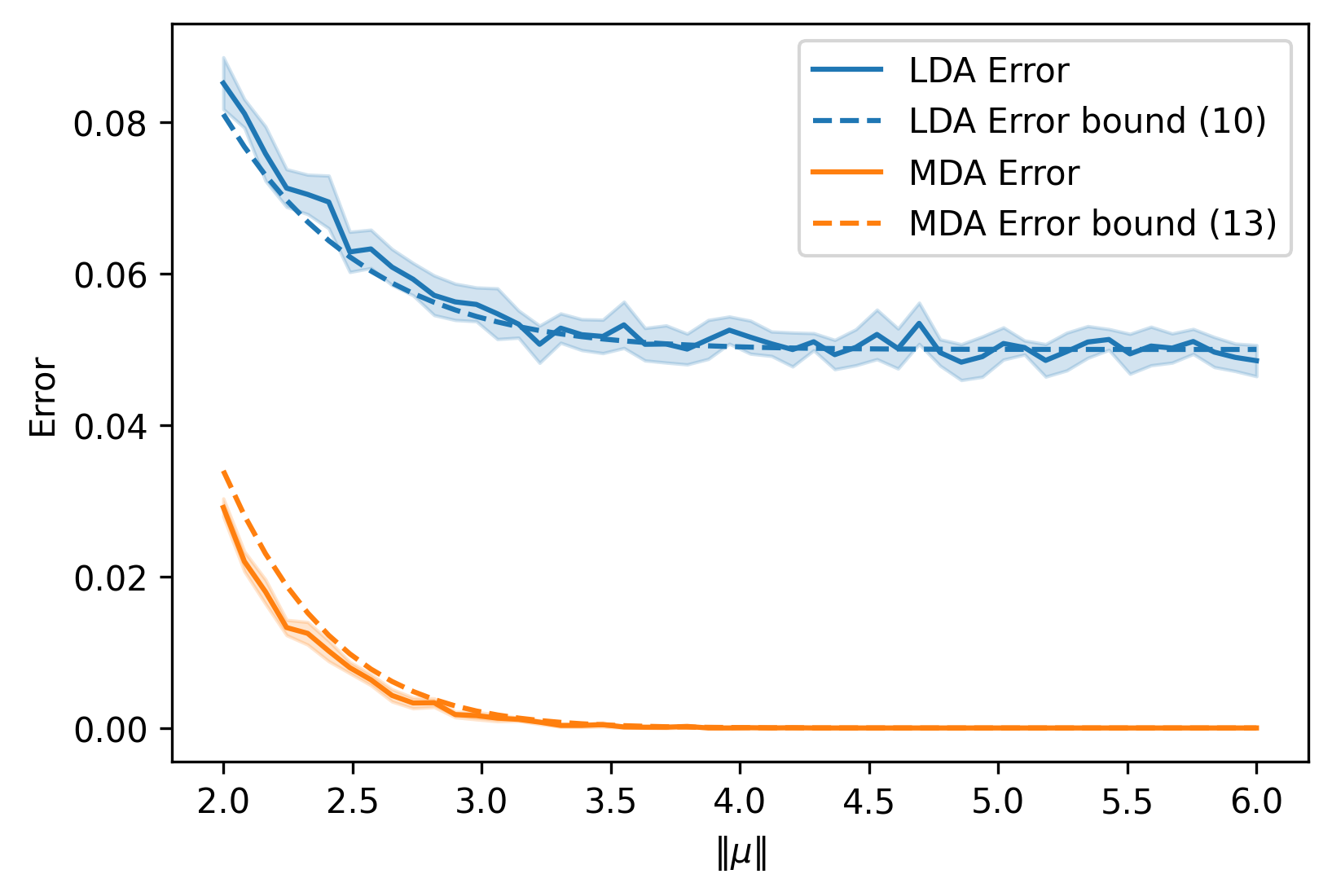}
    \caption{Comparison of empirical errors (solid) and theoretical error bounds (dashed) when $\|\boldsymbol{\mu}\|$ varies. The shaded areas are 95\% confidence bands around the average across 10 runs. The values of the remaining parameters are fixed as follows: $d=50$,  $p=0.9$, $\sigma=1$, $n=7000$, $n_\text{test}=3000$.}
    \label{fig:mu_norm}
    \end{minipage}\hfill\begin{minipage}[t]{.3\textwidth}
    \includegraphics[width=\textwidth]{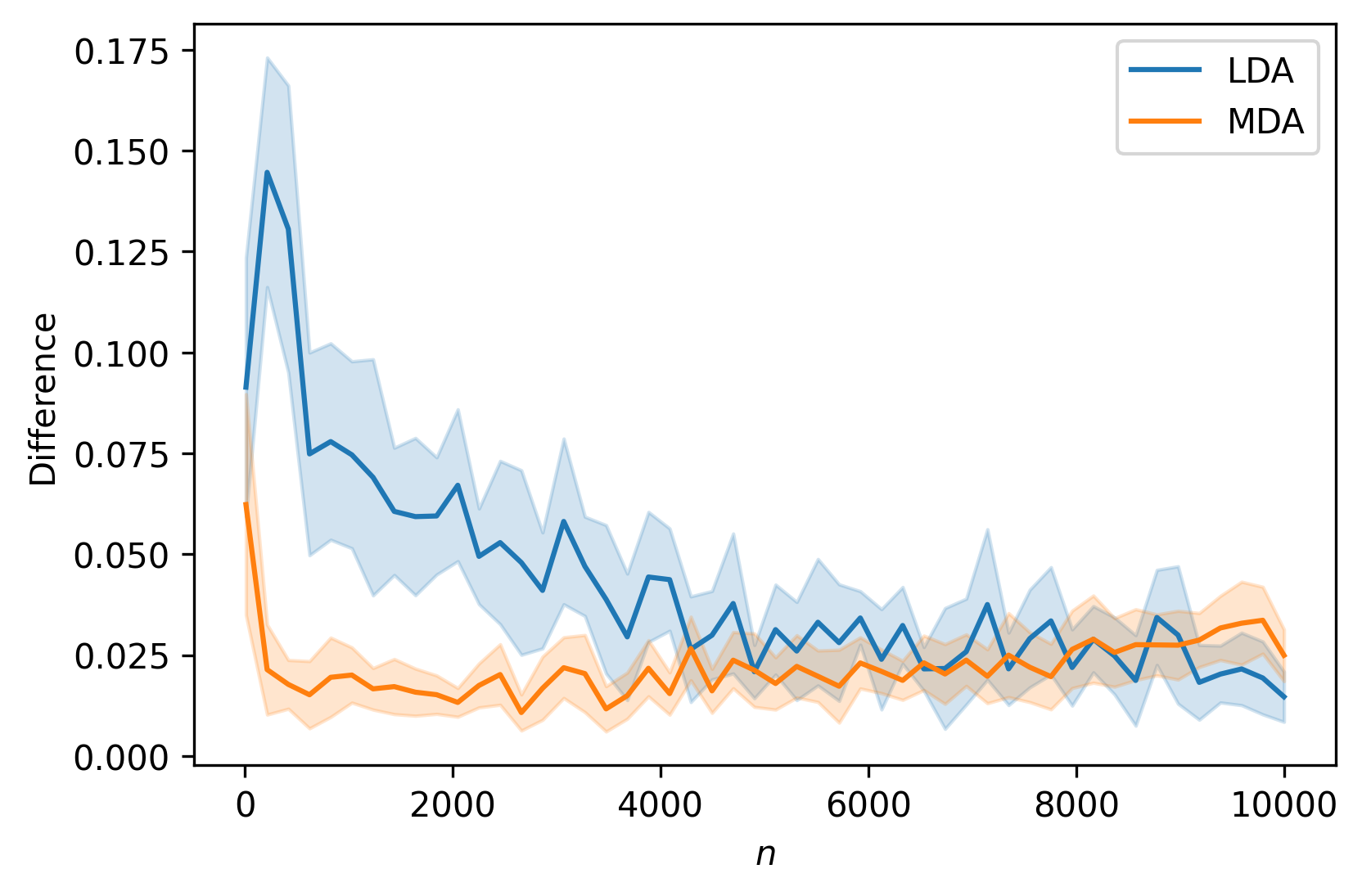}
    \caption{Verifying the order $\widetilde{O}\left(\sqrt{\frac{d}n}\right)$ of the estimation error by plotting $|\text{Test Error} - \text{Error Bound}|\cdot \sqrt{\frac{n}{d\ln n}}$ for increasing values of $n$. The values of the remaining parameters are fixed as follows: $d=50$, $\|\boldsymbol{\mu}\|=2$, $p=0.9$, $\sigma=1$, $n_\text{test}=3000$. The shaded areas are 95\% confidence bands around the average across 10 runs.}
    \label{fig:n}    
    \end{minipage}\hfill\begin{minipage}[t]{.3\textwidth}
    \includegraphics[width=\textwidth]{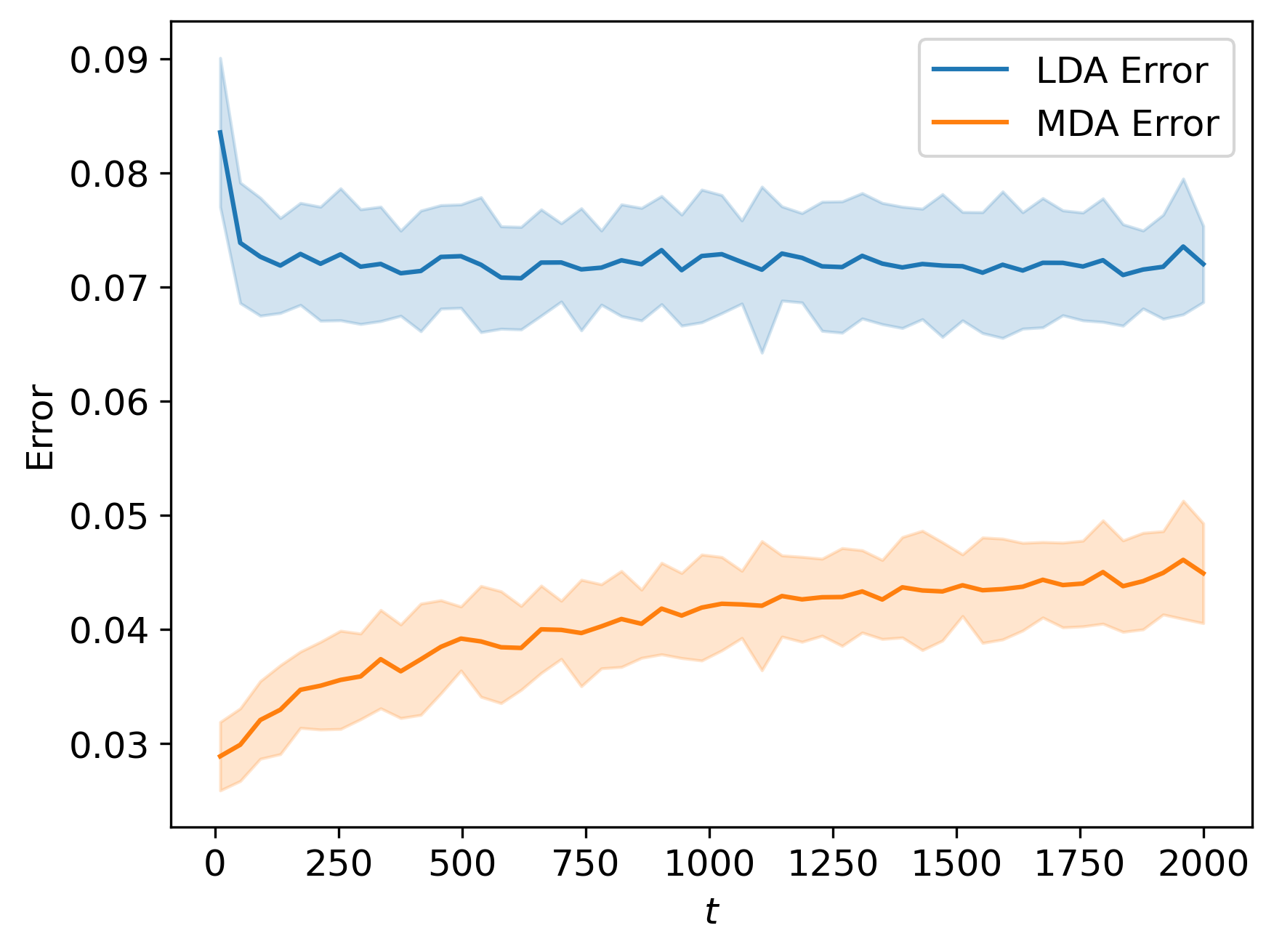}
    \caption{Comparison of the empirical LDA and MDA errors when training on $\mathcal{D}_{1-1/t}$ but testing on $\mathcal{D}_p$, as $t$ grows. The values of the remaining parameters are fixed as follows: $d=50$,  $p=0.9$, $\|\boldsymbol{\mu}\|=2$, $\sigma=1$, $n=10000$, $n_\text{test}=10000$. The shaded areas are 95\% confidence bands around the average across 20 runs.}
    \label{fig:p_q}    
    \end{minipage}
\end{figure*}
\begin{figure*}
    \begin{minipage}[t]{.3\textwidth}
        \includegraphics[width=\textwidth]{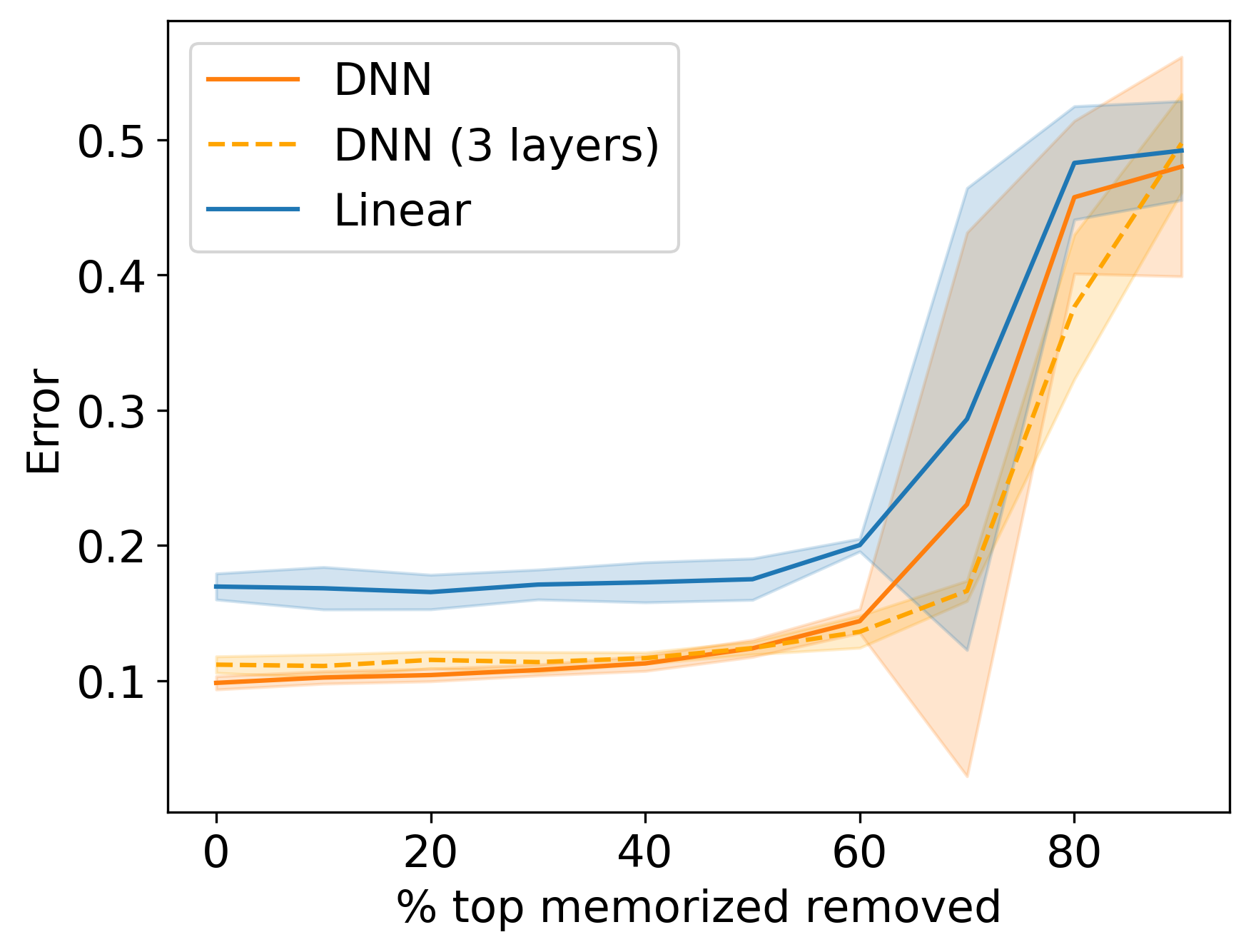}
        \caption{Evaluation of a linear classifier and deep neural networks (Distill-BERT) on a dataset of real movie reviews (SST-2). The shaded areas are 95\% confidence bands around the average across 9 runs per each \% top memorized examples removed.}
        \label{fig:sst2} 
    \end{minipage}\hfill\begin{minipage}[t]{.23\textwidth}
        \includegraphics[width=\textwidth]{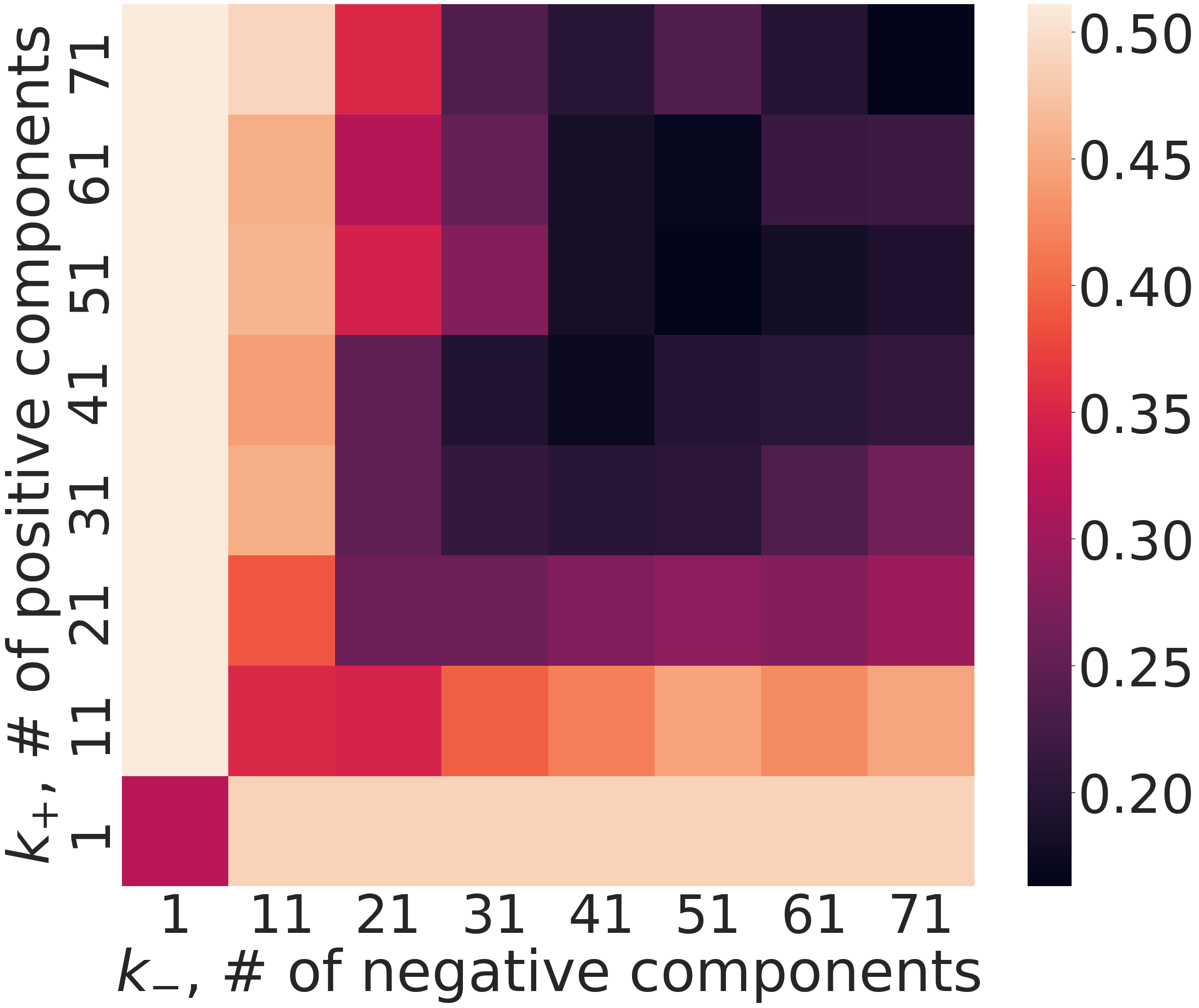}
        \caption{Heatmap of errors when fitting overparameterized MDA classifiers to data generated from our model $\mathcal{D}$, $d=50$, $\boldsymbol\mu=\frac{1}{\sqrt{2}}(1,\ldots,1)$, $n=300$, $\sigma=1$, $p=0.9$.}
        \label{fig:benign}
    \end{minipage}\hfill\begin{minipage}[t]{.39\textwidth}
        \includegraphics[width=.48\textwidth]{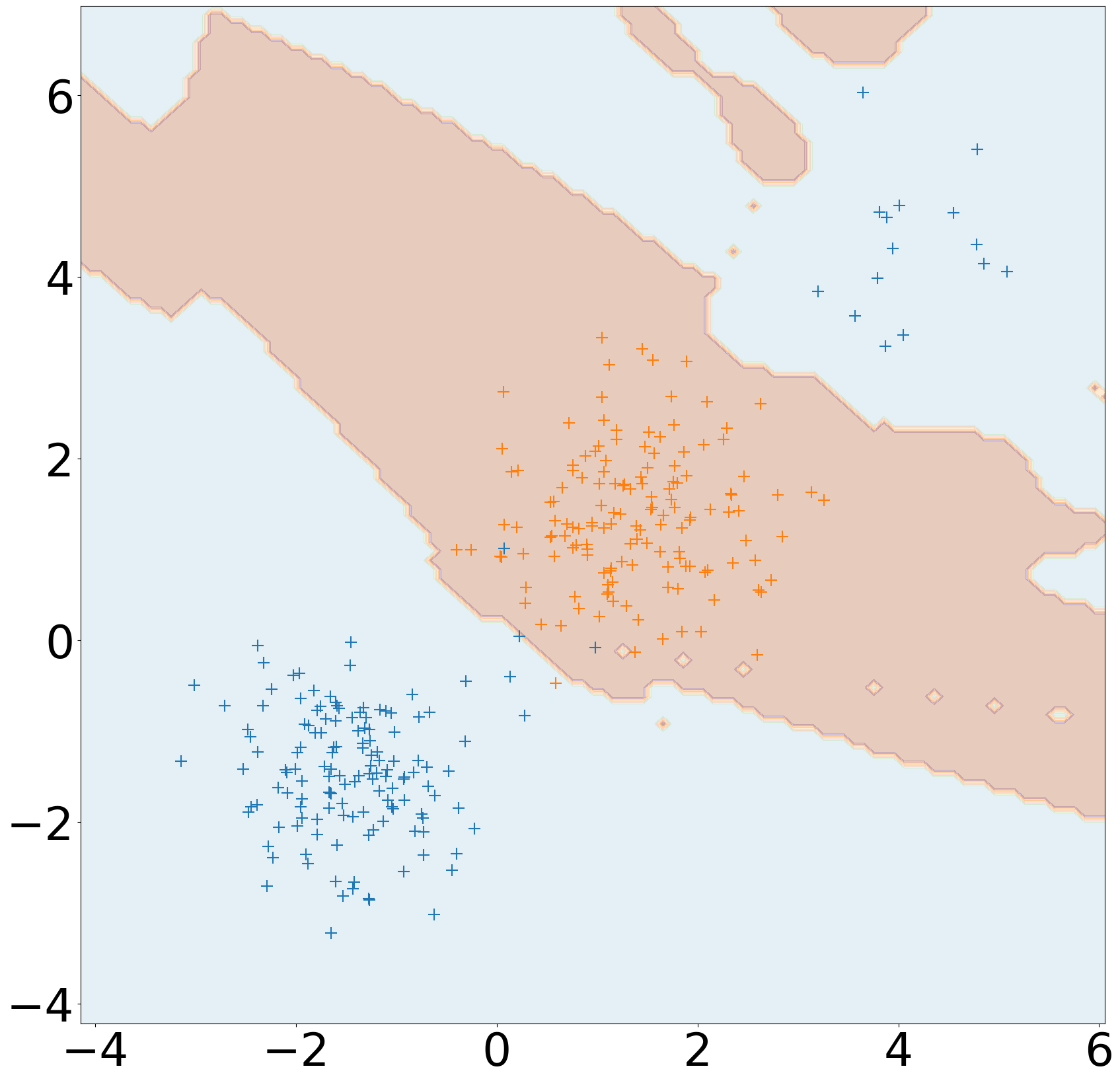}\hfill\includegraphics[width=.48\textwidth]{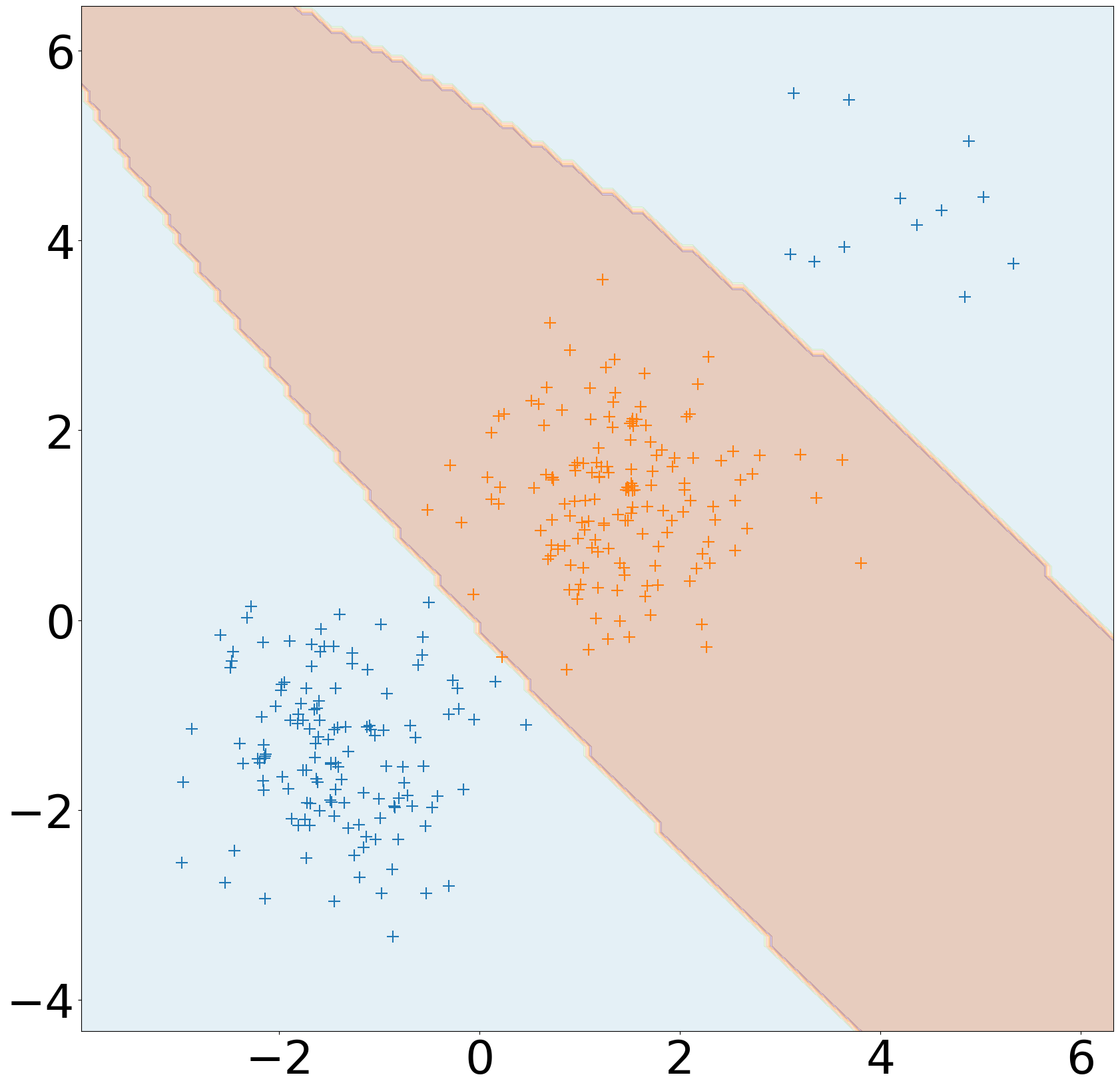}
        \vspace{10pt}
        \caption{Decision boundaries of overparameterized and properly parametrized MDA classifiers. Left: $k_+=k_-=30$, Right: $k_+=1, k_-=2$. The rest parameters are as follows:
        $\boldsymbol\mu=\frac{1}{\sqrt{2}}(1,1)$, $n=300$, $\sigma=1$, $p=0.9$.}
        \label{fig:db_2d}   
    \end{minipage}
\end{figure*}
\iffalse
\begin{figure*}
    \begin{minipage}[t]{.3\textwidth}
        \includegraphics[width=\textwidth]{Figures/SST_2}
        \caption{Evaluation of a linear classifier and deep neural networks (Distill-BERT) on a dataset of real movie reviews (SST-2). Error bars are standard deviations computed from 9 runs per each \% top memorized examples removed.}
        \label{fig:sst2}        
    \end{minipage}\hfill\begin{minipage}[t]{.22\textwidth}
        \includegraphics[width=\textwidth]{Figures/benign}
        \caption{Heatmap of errors when fitting overparameterized MDA classifiers to data generated from our model $\mathcal{D}$, $d=50$, $\boldsymbol\mu=\frac{1}{\sqrt{2}}(1,1)$, $n=300$, $\sigma=1$, $p=0.9$.}
        \label{fig:benign}
    \end{minipage}\hfill\begin{minipage}[t]{.39\textwidth}
        \includegraphics[width=.48\textwidth]{Figures/db_2d.png}\hfill\includegraphics[width=.48\textwidth]{Figures/db_optim.png}
        \caption{Decision boundaries of overparameterized and properly parametrized MDAs. Left: $k_+=k_-=30$, Right: $k_+=1, k_-=2$. The rest parameters are as follows:
        $\boldsymbol\mu=\frac{1}{\sqrt{2}}(1,1)$, $n=300$, $\sigma=1$, $p=0.9$.}
        \label{fig:db_2d}        
    \end{minipage}
\end{figure*}
\fi
As we can see, the empirical errors are generally consistent with our bounds. The LDA test error is statistically close to our LDA error bound \eqref{eq:lda_err_lb}, as the latter is mainly within the 95\% confidence band. Meanwhile, the MDA test error is significantly lower than our MDA error bound \eqref{eq:mda_err_ub}. This is not surprising because, as we already mentioned in Section~\ref{sec:classifiers}, for LDA we derived the exact misclassification error modulo $\widetilde{O}(\sqrt{d/n})$, while for MDA we got the \emph{upper} bound for the error.

\paragraph{Dependence on $p$.} To test the dependence of error bounds on $p$, we vary $p$ in the interval $[0.5, 1]$ with a step 0.01. The results are shown in Figure~\ref{fig:p}.  As we can see, the situation is similar to the previous one: for LDA, the empirical error agrees well with our formula \eqref{eq:lda_err_lb}, especially for $p$ closer to 1; while for MDA, in most cases, it is significantly below our bound \eqref{eq:mda_err_ub}. 

\paragraph{Dependence on $n$.}
To check the correctness of the order $\widetilde{O}(\sqrt{d/n})$ of the estimation error, we consider the expression
\begin{equation}
    |\text{Test Error} - \text{Error Bound}|\cdot \sqrt{\frac{n}{d\ln n}}\label{eq:expr_n}
\end{equation}
for increasing $n$. The results are shown in Figure~\ref{fig:n}. Here we observe the boundedness of the expression \eqref{eq:expr_n}, which confirms that the order $\widetilde{O}(\sqrt{d/n})$ of the estimation error is correct.

\paragraph{Training on $\mathcal{D}_q$, but Testing on $\mathcal{D}_p$.} Finally, for experimental verification of the conclusions from Theorem~\ref{thm:train_test_diff}, we generate training samples from $D_{1-1/t}$, and test samples from $D_p$. We vary $t$ in the interval $[10, 2000]$  and observe the behavior of the LDA and MDA test errors. The results are shown in Figure~\ref{fig:p_q}.
%\begin{figure}[htbp]
%    \centering
%    \includegraphics[width=.4\textwidth]{example-image-a}
%    \caption{Training on $\mathcal{D}_{1-1/t}$, but Testing on $\mathcal{D}_p$.}
%    \label{fig:p_q}
%\end{figure}
As we can see, the empirical error curves agree with the predictions of our Theorem~\ref{thm:train_test_diff}. Namely, the LDA error exceeds the MDA error, and the gap between the two remains feasible with the increase of $t$. %Both errors practically coincide at large $t$ (i.e., in the absence of a long tail in the training sample).

\subsection{Real Data}
We conduct our experiments\footnote{Our computing infrastructure for these experiments is as follows. CPU: Intel Core i9-10900X CPU @ 3.70GHz, GPU: 2$\times$ Nvidia RTX 3090, RAM: 128 Gb, Operating System: Ubuntu 22.04.1 LTS, torch: 1.13.1, cuda: 11.7, pandas: 1.5.3.} on SST-2 \cite{socher-etal-2013-recursive}, which is a dataset for sentence-level binary (positive vs. negative) sentiment classification. It has 6920 training, 872 validation, and 1821 test examples. We use the pre-trained Distill-BERT model \cite{DBLP:journals/corr/abs-1910-01108} that consists of 6 transformer layers, where each layer is composed of 12 attention heads. We take the  representation of the \texttt{[CLS]} token from the 6$^\text{th}$ layer  for classification. We train the network with Adam, setting the learning rate and batch size to ${10}^{-6}$ and 100, respectively. 

Calculating the memorization score according to \eqref{eq:mem_score} requires retraining the network for every training example (that is being removed) which is not computationally feasible. Thus we approximate the memorization score using the method of Zheng and Jiang \cite{DBLP:conf/acl/ZhengJ22}.

We finetune two versions of the pre-trained Distill-BERT on SST-2: in one we freeze all layers except the top classification layer, in the second we finetune all layers. The first model---called Linear---is essentially a linear classifier (logistic regression), which receives text representations from the 6th layer of Distill-BERT as input, and the representations are not trained. It is clear that such a model is not capable of memorizing rare atypical examples. The second model---called DNN---is a full-fledged deep neural network that has enough capacity (66 million parameters) to memorize atypical examples from rare subpopulations. We also consider an intermediate option---called DNN (3 layers)---when the three lower layers are frozen, and the remaining layers are finetuned.

The experiment is as follows: (1) we compute the memorization score for each training example through DNN,\footnote{This means that in the definition of the memorization score \eqref{eq:mem_score}, we use a DNN trained by gradient-based method as a learning algorithm $A$. However, we \emph{approximate} Eq.~\ref{eq:mem_score} via the method of Zheng and Jiang \cite{DBLP:conf/acl/ZhengJ22} to avoid repeated retraining for each example.
} (2) remove $m$\% of the top memorized examples from the training set, (3) train the Linear and DNN models on such a set with hyperparameters tuned on the validation set, (4) and finally evaluate both models on the test set. The results of this experiment for different $m$ are shown in Figure~\ref{fig:sst2}. As we can see, the error of the Linear classifier is always greater than the DNN error. Further, when the tail is shortened (i.e., when a larger number of top memorized examples from training are discarded), the gap between the errors of Linear and DNN slightly decreases. This is consistent with the predictions of our theory, albeit built with simpler assumptions on the distribution of data and for more interpretable classifiers. 

At certain point, the difference between the errors is not statistically significant. This happens because at a high percentage of removal, examples are removed not only from the tail, but also from the main subpopulations, which sharply worsens the performance of both the Linear and DNN classifiers. Note that this regime is not considered in our theory.

A careful reader may notice that the DNN error is not close to zero, even when no examples are removed from the training set. This is because there are examples in the test set that the DNN cannot classify correctly, even though it fits the training set perfectly. In principle, such difficult examples can be simulated in our data-generating model by introducting label flipping noise, and we defer such modification to our future work.

\section{Discussion on Benign Overfitting}

\iffalse
As is already clear, the departure point for our work is Feldman's paper \cite{DBLP:conf/stoc/Feldman20} on the long-tail theory. One of the  results in that paper is the lower bound for a learning algorithm $A$ operating on a sample $S$ from a distribution $\mathcal{P}$ and producing a hypothesis $h$. It can be written as
$$
\E_{S\sim\mathcal{P}^n,\, h\leftarrow A(S)}\left[\e_{\mathcal{P}}[h]\right]\ge\mathrm{opt}_\mathcal{D}+\tau_1\cdot\E[\mathrm{errn}_S(A,1)],
$$
where $\mathrm{opt}_\mathcal{P}$ is the minimum achievable error by any algorithm, and $\mathrm{errn}_S(A,1)$ is the number of examples that appear once in the dataset $S$ and are mislabeled by the classifier $h$ that $A$ outputs on $S$. The coefficient $\tau_1$ quantifies the cost one needs to pay for \emph{every} training example that the learner does not fit. It can be lowerbounded by the total weight of the part of the prior distribution over subpopulations which has frequency on the order of $1/n$, and also the absence of frequencies on this order implies negligible $\tau_1$. Roughly speaking, the heavier the tail, the higher is the price of not fitting examples from such a tail.
\fi

%However, if we have a learner which is able to fit all (or most) of the training examples (including rare and atypical ones), then where is the guarantee that the overfitting will be benign and not malignant? Thus, 

As was already mentioned in the Introduction, Feldman's long-tail theory explains the \emph{need} for overfitting, but does not explain how exactly modern overparameterized learning algorithms manage to overfit without harming generalization \cite{DBLP:conf/iclr/ZhangBHRV17}. %Therefore, an important component of benign overfitting is the ability of the learning algorithm to choose a sufficiently complex, but not too complex, hypothesis from the class of functions being trained. 
Despite the availability of the answer for some model classes  \cite{DBLP:journals/jmlr/ChatterjiL21,DBLP:conf/colt/Shamir22,DBLP:conf/colt/FreiCB22,cao2022benign,DBLP:conf/icassp/WangT21}, our main concern with the current trend in theoretical studies of benign overfitting is in the assumptions about the data generating process, namely that only one subpopulation is allowed in each class, and linear inseparability is achieved by introducing random label-flipping noise. We repeat once again that such a setup does not fit in with Feldman's long-tail theory, in which the presence of rare subpopulations in classes is a prerequisite. Without this condition, there is no need for overfitting. Therefore, we would like to draw the attention of the learning theory community to the existing gap between theoretical setups and reality regarding the distribution of data. %So, for example, for a sample from Figure~\ref{fig:train_set}, the (almost) linear classifier from steps 10 and 100 of Figure~\ref{fig:sgd_progression} is already optimal in terms of generalizing ability to new data. At the same time, such a classifier does \emph{not} overfit (memorize) noisy examples from the training set.

%However, the problem with the current trend in theoretical studies of benign overfitting \cite{DBLP:journals/jmlr/ChatterjiL21,DBLP:conf/colt/Shamir22,DBLP:conf/colt/FreiCB22,cao2022benign,DBLP:conf/icassp/WangT21} is that the analysis is carried out under the assumption of a single (possibly noisy) subpopulation per class. Therefore, we would like to draw the attention of the learning theory community to the existing gap between theoretical setups and reality regarding the distribution of data.

In the theoretical part of our work, we do not deal with overparameterized models like deep neural networks. Our ``complex'' classifier (MDA) is actually only as complex as the data requires. Therefore, we cannot claim to have shown benign overfitting under the conditions of our data-generating model. We have shown the underfitting of the linear classifier (LDA) and the \emph{proper} fitting of the MDA classifier with the right number of components.

However, we are curious about what  happens if we give the MDA classifier the ability to fit many more Gaussians than necessary. To do this, we conduct the following experiment. The data is generated from the same model $\mathcal{D}$  that we used earlier (we fixed $d=50$, $\boldsymbol\mu=\frac{2}{\sqrt{d}}(1,\ldots,1)$, $n=300$, $\sigma=1$, $p=0.9$). For data points from the positive class, we fit $k_+$ Gaussians, and for points from the negative class, we fit $k_-$ Gaussians. Since in this case, we have no simple way to estimate the parameters by the method of moments, all parameters are estimated by the approximate maximum likelihood method through the EM algorithm. Let ${f}_+(\mathbf{x})$ and ${f}_-(\mathbf{x})$ be the resulting estimated p.d.f.'s for the positive and negative classes, respectively. Then the MDA classifier can be written as 
$$
h^\text{MDA}(\mathbf{x};k_+,k_-)=\begin{cases}+1,\quad&\text{if }f_+(\mathbf{x})\ge f_-(\mathbf{x})\\
-1&\text{otherwise}
\end{cases}
$$
We vary $k_+$ and $k_-$ in the interval $[1,71]$ with a step $10$ and calculate the classifier error on the test sample. The results of this experiment are shown in Figure~\ref{fig:benign}, which is a heatmap of test errors for different pairs $(k_+, k_-)$. The training error is zero for all pairs $(k_+, k_-)$, except for $k_+=k_-=1$.
%\begin{figure}[htbp]
%    \centering
%    \includegraphics[width=.4\textwidth]{Figures/benign}
%    \caption{Fitting overparameterized MDA classifiers to data generated from $\mathcal{D}$.}
%    \label{fig:benign}
%\end{figure}
Notably, there is a clear pattern: when $k_+$ and $ k_-$ are close to each other, the performance can be better than when there is a heavy imbalance between $k_+$ and $k_-$.

To understand how an overparametrized MDA classifier manages to overfit benignly, we plot a decision curve for the case $d=2$, $k_+=k_-=30$ (Figure~\ref{fig:db_2d}, left). As we can see, despite the potential to overfit malignantly with a complex decision curve, the EM algorithm chooses a fairly simple classifier that is not so different from the optimal one (Figure~\ref{fig:db_2d}, right), that uses $k_+=1$, $k_-=2$.

%Thus, we see that benign overfitting is possible for an overparameterized MDA classifier within our data generating model. Natural questions arise: Under what conditions is overfitting benign? What role does the EM algorithm play in achieving benign overfitting? The last question is motivated by the fact that in gradient-based learning, SGD plays an important role in achieving benign overfitting \cite{DBLP:conf/nips/KalimerisKNEYBZ19}. We defer these questions to our future work. However, 
It is noteworthy that an overparameterized MDA classifier is able to overfit beningly on data generated from our model, because such a framework is more interpretable and amenable to analysis than overparameterized deep neural networks trained on real data. Accordingly, it becomes possible to study the phenomenon of benign overfitting in a simplified setting without linking it to deep learning, in which it is usually considered.

\iffalse
\section{Related Work}

We would also like to mention the works of Raunak~et~al.~\cite{raunak-etal-2021-curious} and Zheng~and~Jiang \cite{DBLP:conf/acl/ZhengJ22}, in which Feldman's long-tail theory is empirically validated for machine translation and text classification tasks, respectively. These works motivated us to develop our data-generating model.
\fi

\section{Conclusion}
In this work we have focused on building an \emph{interpretable} mathematical framework for the analysis of learning algorithms capable of memorizing rare/atypical examples that usually occur in natural data, such as texts and images. The key point in our work is the data-generating model based on Gaussian mixtures, which demonstrates the inability of a simple classifier without sufficient memory to correctly label rare and atypical test examples. At the same time, for a more complex (but not too complex) classifier with sufficient memory, the near-to-optimal generalization ability is shown. Moreover, the dynamics of the performance of these classifiers with tail shortening has been studied both theoretically and experimentally, and the experiments were carried out both on synthetic and real data. 

The last but not least property of our framework is that it allows for benign overfitting, and this is what we plan to study in the near future.
In this regard, it will be interesting to analyze the behavior of overparameterized learning algorithms (such as MDA with a redundant number of components, deep neural networks, and nearest-neighbor classifiers) on data generated from our model. This will require obtaining new results in terms of sufficient conditions for benign overfitting to happen under the assumptions of our model.

%Another interesting direction is the analysis of learning algorithms equipped with a retrieval mechanism within our data generating model. Empirically, such hybrid models with explicit memory demonstrate more efficient benign overfitting on language processing tasks than vanilla neural networks \cite{DBLP:conf/iclr/KhandelwalLJZL20}. Therefore, we consider it reasonable to analyze such models within our framework.

\ack This research has been funded by Nazarbayev University under Faculty-development competitive research grants program for 2023-2025 Grant \#20122022FD4131, PI R. Takhanov. Igor Melnykov's work on this project was supported by a Fulbright US Scholar Grant administered by the US Department of State Bureau of Educational and Cultural Affairs (grant ID: PS00334837). The authors would like to thank Christopher Dance for a thorough review of our work (including mathematical proofs), Matthias Gall\'e for his constructive feedback, including the suggestion to add a discussion on benign overfitting. We would like to thank the reviewers for their valuable feedback, in particular Reviewer 1 for a deep reading of our work, Reviewers 2 and 3 for carefully reading our response, Reviewer 6 for good questions that helped improve the presentation of the material.

\bibliography{ref}

\clearpage
\onecolumn

\appendix

\section{Proofs}

\subsection{Estimating $\boldsymbol\mu$ by the method of moments}\label{sec:mu_mom}
\begin{lem}
Let $\{(\mathbf{X}_i,Y_i)\}_{i=1}^n$ be a sample of size $n$ from the distribution $\mathcal{D}$ specified in \eqref{eq:prior}--\eqref{eq:neg_c2}. Then estimating $\boldsymbol\mu$ by the method of moments gives
\begin{equation}
    \hat{\boldsymbol\mu}=\frac{1}{2n(1-p)}\sum_{i=1}^n \mathbf{X}_i.\label{eq:mu_hat}
\end{equation}
\end{lem}
\begin{proof}
\begin{align*}
\E\left[\sum_{i=1}^n \mathbf{X}_i\right]&=\E\left[\sum_{i=1}^n \mathbf{X}_i\cdot\mathbb{I}[Y_i=+1]\right]+\E\left[\sum_{i=1}^n \mathbf{X}_i\cdot\mathbb{I}[Y_i=-1,K=1]\right]+\E\left[\sum_{i=1}^n \mathbf{X}_i\cdot\mathbb{I}[Y_i=-1,K=2]\right]\\
&=\sum_{i=1}^n\left[\E[\mathbf{X}_i\mid Y_i=+1]\cdot\frac12+\E[\mathbf{X}_i\mid Y_i=-1,K=1]\cdot\frac{p}{2}+\E[\mathbf{X}_i\mid Y_i=-1,K=2]\cdot\frac{1-p}{2}\right]\\
&=\frac12\sum_{i=1}^n[\boldsymbol\mu+(-\boldsymbol\mu)p+(3\boldsymbol\mu)(1-p)]=\frac{n}{2}(1-p)(4\boldsymbol\mu).
\end{align*}
Solving for $\boldsymbol\mu$ and replacing $\E\left[\sum_{i=1}^n\mathbf{X}_i\right]$ by $\sum_{i=1}^n\mathbf{X}_i$ concludes the proof.
\end{proof}

\begin{lem}\label{lem:mom_conc}
Let $\hat{\boldsymbol\mu}$ be the method of moments estimator of $\boldsymbol\mu$ given by \eqref{eq:mu_hat}.   Then
\begin{equation}
\Pr_{S\sim\mathcal{D}^n}[\|\hat{\boldsymbol\mu}-\boldsymbol\mu\|>t]\le e^{-\Omega\left(\frac{n t^2}{d}\right)},\label{eq:mom_conc}
\end{equation}
\end{lem}
\begin{proof}
    Let $(\mathbf{X},Y,K)\sim\mathcal{D}$. The moment-generating function (m.g.f.) of $\mathbf{X}$ is
    \begin{equation}
        M_{\mathbf{X}}(\mathbf{t})=\E\left[e^{\mathbf{t}^\top\mathbf{X}}\right]=\E\left[\E\left[e^{\mathbf{t}^\top\mathbf{X}}\bigm| Y,K\right]\right].\label{eq:mgf_cond}
    \end{equation}
    Recall, that the m.g.f. of a multivariate Gaussian $\mathcal{N}(\boldsymbol\mu,\mathbf{A})$ is $e^{\mathbf{t}^\top\boldsymbol\mu+\frac12\mathbf{t}^\top\mathbf{At}}$. Thus,
    \begin{align}
        &\E\left[e^{\mathbf{t}^\top\mathbf{X}}\bigm|Y=1\right]=e^{\mathbf{t}^\top\boldsymbol\mu+\frac12\mathbf{t}^\top\sigma^2\mathbf{I}\mathbf{t}}=e^{\mathbf{t}^\top\boldsymbol\mu+\frac{\sigma^2\|\mathbf{t}\|^2}{2}},\notag\\
        &\E\left[e^{\mathbf{t}^\top\mathbf{X}}\bigm|Y=-1,K=1\right]=e^{-\mathbf{t}^\top\boldsymbol\mu+\frac{\sigma^2\|\mathbf{t}\|^2}{2}},\notag\\
        &\E\left[e^{\mathbf{t}^\top\mathbf{X}}\bigm|Y=-1,K=2\right]=e^{3\mathbf{t}^\top\boldsymbol\mu+\frac{\sigma^2\|\mathbf{t}\|^2}{2}}.\label{eq:mgfs}
    \end{align}
    Taking into account \eqref{eq:prior}, \eqref{eq:neg_c1}, and \eqref{eq:neg_c2}, from \eqref{eq:mgf_cond} and \eqref{eq:mgfs}, we have
    \begin{equation}
        M_{\mathbf{X}}(\mathbf{t})=e^{\frac{\sigma^2\|\mathbf{t}\|^2}{2}}\left(\frac12\cdot e^{\mathbf{t}^\top\boldsymbol\mu}+\frac{p}2\cdot e^{-\mathbf{t}^\top\boldsymbol\mu}+\frac{1-p}{2}\cdot e^{3\mathbf{t}^\top\boldsymbol\mu}\right).
    \end{equation}
    Now, for a sample $\{(\mathbf{X}_i,Y_i)\}_{i=1}^n\,\,{\stackrel{\text{iid}}{\sim}}\,\,\mathcal{D}$ and $\hat{\boldsymbol\mu}$ defined by \eqref{eq:mu_hat}, we have
    \begin{align}
        &M_{\hat{\boldsymbol\mu}-\boldsymbol\mu}(\mathbf{t})=\E\left[\mathbf{t}^\top(\hat{\boldsymbol\mu}-\boldsymbol\mu)\right]=\E\left[e^{\mathbf{t}^\top\left(\frac{1}{2(1-p)n}\sum_{i=1}^n\mathbf{X}_i-\boldsymbol\mu\right)}\right]=\E\left[e^{\sum_{i=1}^n\frac{\mathbf{t}^\top}{2(1-p)n}\mathbf{X}_i}\right]e^{-\mathbf{t}^\top\boldsymbol\mu}\notag\\
        &=\left(\E\left[e^{\frac{\mathbf{t}^\top}{2(1-p)n}\mathbf{X}}\right]\right)^n e^{-\mathbf{t}^\top\boldsymbol\mu}=\left[M_\mathbf{X}\left(\frac{\mathbf{t}}{2(1-p)n}\right)\right]^n e^{-\mathbf{t}^\top\boldsymbol\mu}\notag\\
        &=\left[\exp\left({\frac{\sigma^2\|\mathbf{t}\|^2}{8(1-p)^2 n^2}}\right)\left(\frac12\exp\left(\frac{\mathbf{t}^\top\boldsymbol\mu}{2(1-p)n}\right)+\frac{p}2\exp\left(-\frac{\mathbf{t}^\top\boldsymbol\mu}{2(1-p)n}\right)+\frac{1-p}2\exp\left(\frac{3\mathbf{t}^\top\boldsymbol\mu}{2(1-p)n}\right)\right)\right]^n \exp\left(-\mathbf{t}^\top\boldsymbol\mu\right)\notag\\
        &{\stackrel{(\ast)}{=}}\exp\left(\frac{\sigma^2\|\mathbf{t}\|^2}{8(1-p)^2 n}\right)\left[\frac12+\frac12\cdot\frac{\mathbf{t}^\top\boldsymbol\mu}{2(1-p)n}+\frac{p}2-\frac{p}2\cdot\frac{\mathbf{t}^\top\boldsymbol{\mu}}{2(1-p)n}+\frac{1-p}2+\frac{1-p}2\cdot\frac{3\mathbf{t}^\top\boldsymbol\mu}{2(1-p)n}+O\left(\frac{(\mathbf{t}^\top\boldsymbol\mu)^2}{(1-p)^2n^2}\right)\right]^n \exp\left(-\mathbf{t}^\top\boldsymbol\mu\right)\notag\\
        &\le\exp\left(\frac{\sigma^2\|\mathbf{t}\|^2}{8(1-p)^2 n}\right)\underbrace{\left[1+\frac1{n}\left({\mathbf{t}^\top\boldsymbol\mu}+O\left(\frac{\|\mathbf{t}\|^2\|\boldsymbol\mu\|^2}{(1-p)^2n}\right)\right)\right]^n}_{\le\exp\left(\mathbf{t}^\top\boldsymbol\mu+O\left(\frac{\|\mathbf{t}\|^2\|\boldsymbol\mu\|^2}{(1-p)^2n}\right)\right)}\exp\left(-\mathbf{t}^\top\boldsymbol\mu\right)\le\exp\left(O\left(\frac{\|\mathbf{t}\|^2(\sigma^2+\|\boldsymbol\mu\|^2)}{(1-p)^2 n}\right)\right),\label{eq:mgf_bound}
    \end{align}
    where in $(\ast)$ we use the Taylor expansion $e^x=1+x+O(x^2)$ which is valid for small enough $x$ (i.e. large enough $n$). Since $\sigma$, $\boldsymbol{\mu}$, and $p$ are constants, \eqref{eq:mgf_bound} implies that $\hat{\boldsymbol\mu}-\boldsymbol\mu$ is a subGaussian random vector with variance proxy $\Theta\left(\frac{1}{n}\right)$, which in turn implies that it is norm-subGaussian with variance proxy $\Theta\left(\frac{d}{n}\right)$ \cite[Lemma 1]{DBLP:journals/corr/abs-1902-03736}. Thus, \eqref{eq:mom_conc} follows from the definition of a norm-subGaussian random vector \cite[Definition 3]{DBLP:journals/corr/abs-1902-03736}.
\end{proof}

\subsection{Proof of Lemma~\ref{lem:lda_err}}\label{sec:lda_err_proof}

Since LDA fits only one Gaussian to the negative class, we first find its center:
\begin{align}
    %\E[X\mid Y=+1]&=\mu=:\mu_+\\
    \E[\mathbf{X}\mid Y=-1]&=\E[\mathbf{X}\mid Y=-1,K=1]\cdot p+\E[\mathbf{X}\mid Y=-1,K=2]\cdot(1-p)\notag\\
    &=-\boldsymbol\mu p + 3\boldsymbol\mu(1-p)=-(4p-3)\boldsymbol\mu=:\boldsymbol\mu_-.\label{eq:mean_neg}
\end{align}
Further, LDA assumes that the Gaussians corresponding to the negative and positive classes have the same (diagonal) covariance matrix. Let us denote it by $\sigma^2_\text{LDA}\mathbf{I}$. The precise value of $\sigma^2_\text{LDA}$ is not important for our analysis as it is canceled anyway, so we will not compute it here. As was mentioned in Section~\ref{sec:classifiers}, the LDA classifier that has access to the true values of the parameters can be written as
\begin{equation}
h^\text{LDA}(\mathbf{x})=\begin{cases}
+1\quad&\text{if }f(\mathbf{x};\boldsymbol\mu,\sigma^2_{\text{LDA}}\mathbf{I})\ge f(\mathbf{x};\boldsymbol\mu_-,\sigma^2_{\text{LDA}}\mathbf{I})\\
-1&\text{otherwise}\label{eq:lda_clf2}
\end{cases},
\end{equation}
where
\begin{equation}
f(\mathbf{x};\boldsymbol\mu,\boldsymbol\Sigma):=\frac{1}{\sqrt{|\boldsymbol{\Sigma}|(2\pi)^{n}}}\exp\left(-\frac12(\mathbf{x}-\boldsymbol\mu)^\top\boldsymbol{\Sigma}^{-1}(\mathbf{x}-\boldsymbol{\mu})\right).\label{eq:gauss_pdf}
\end{equation}
We first find the misclassification error of the ``ideal'' LDA given by \eqref{eq:lda_clf2}, and then we discuss what happens when the parameter $\boldsymbol\mu$ is unknown and is estimated from data. We have
\begin{align}
\e_\mathcal{D}[h^\text{LDA}]:=\Pr_{\mathbf{X},Y\sim\mathcal{D}}[h^\text{LDA}(\mathbf{X})\ne Y]&=\Pr[h^\text{LDA}(\mathbf{X})\ne Y\mid Y=+1]\cdot\Pr[Y=+1]\notag\\
&+\Pr[h^\text{LDA}(\mathbf{X})\ne Y\mid Y=-1,K=1]\cdot\Pr[Y=-1,K=1]\notag\\
&+\Pr[h^\text{LDA}(\mathbf{X})\ne Y\mid Y=-1,K=2]\cdot\Pr[Y=-1,K=2].\label{eq:lda_err_decomp}
\end{align}
We compute the three conditional probabilities separately. For the positive class we have
\begin{align}
    \Pr[h^\text{LDA}(\mathbf{X})\ne Y\mid Y=+1]&=\Pr[f(\mathbf{X};\boldsymbol\mu,\sigma^2_{\text{LDA}}\mathbf{I})<f(\mathbf{X};\boldsymbol\mu_-,\sigma^2_{\text{LDA}}\mathbf{I})\mid Y=+1]\notag\\
    &=\Pr[-\|\mathbf{X}-\boldsymbol\mu\|^2<-\|\mathbf{X}-\boldsymbol\mu_-\|^2\mid Y=+1]\notag\\
    &=\Pr[(\boldsymbol\mu-\boldsymbol\mu_-)^\top(2\mathbf{X}-\boldsymbol\mu_--\boldsymbol\mu)<0\mid Y=+1]\notag\\
    &=\{\text{representing }\mathbf{X}=\boldsymbol\mu+\sigma\mathbf{I}^{1/2}\mathbf{Z}\}=\Pr_{\mathbf{Z}\sim\mathcal{N}(\mathbf{0},\mathbf{I})}\left[(\boldsymbol\mu-\boldsymbol\mu_-)^\top(\boldsymbol\mu-\boldsymbol\mu_-+2\sigma\mathbf{Z})<0\right]\notag\\
    &=\Pr_{\mathbf{Z}\sim\mathcal{N}(\mathbf{0},\mathbf{I})}\left[\|\boldsymbol\mu-\boldsymbol\mu_{-}\|^2+2\sigma(\boldsymbol\mu-\boldsymbol\mu_-)^\top\mathbf{Z}<0\right]\notag\\
    &=\Bigl\{\text{since }(\boldsymbol\mu-\boldsymbol\mu_-)^\top\mathbf{Z}\sim\mathcal{N}\Bigl(0,\|\boldsymbol\mu-\boldsymbol\mu_-\|^2\Bigr)\Bigr\}=\Pr_{Z\sim\mathcal{N}(0,1)}\left[\|\boldsymbol\mu-\boldsymbol\mu_-\|^2+2\sigma\|\boldsymbol\mu-\boldsymbol\mu_-\|Z<0\right]\notag\\
    &=\Pr_{Z\sim\mathcal{N}(0,1)}\left[Z<-\frac{\|\boldsymbol\mu-\boldsymbol\mu_-\|}{2\sigma}\right].\label{eq:lda_err_pos}
\end{align}
From \eqref{eq:mean_neg} and \eqref{eq:lda_err_pos}, we get
\begin{equation}
    \Pr[h^\text{LDA}(\mathbf{X})\ne Y\mid Y=+1]=\Pr_{Z\sim\mathcal{N}(0,1)}\left[Z<-(2p-1)\frac{\|\boldsymbol\mu\|}{\sigma}\right]=\Phi\left(-(2p-1)\frac{\|\boldsymbol\mu\|}{\sigma}\right).\label{eq:lda_err_one}
\end{equation}
Similarly, we can show that
\begin{align}
    \Pr[h^\text{LDA}(\mathbf{X})\ne Y\mid Y=-1,K=1]&=\Phi\left(-(3-2p)\frac{\|\boldsymbol\mu\|}{\sigma}\right),\label{eq:lda_err_two}\\
    \Pr[h^\text{LDA}(\mathbf{X})\ne Y\mid Y=-1,K=2]&=\Phi\left((2p+1)\frac{\|\boldsymbol\mu\|}{\sigma}\right).\label{eq:lda_err_three}
\end{align}
From \eqref{eq:prior}, \eqref{eq:neg_c1}, and \eqref{eq:neg_c2}, it follows that
\begin{align}
    &\Pr[Y=-1, K=1]=\Pr[K=1\mid Y=-1]\cdot\Pr[Y=-1]=\frac{p}{2},\notag\\
    &\Pr[Y=-1, K=2]=\Pr[K=2\mid Y=-1]\cdot\Pr[Y=-1]=\frac{1-p}{2}.\label{eq:neg_priors}
\end{align}
Now, combining \eqref{eq:lda_err_decomp}, \eqref{eq:lda_err_one}, \eqref{eq:lda_err_two}, \eqref{eq:lda_err_three}, \eqref{eq:prior}, and \eqref{eq:neg_priors} we have
\begin{equation}
\e_\mathcal{D}[h^\text{LDA}]=\frac12\left[\Phi\left(-(2p-1)\frac{\|\boldsymbol\mu\|}{\sigma}\right)\right.\\
    \left.+p\Phi\left(-(3-2p)\frac{\|\boldsymbol\mu\|}{\sigma}\right)+(1-p)\Phi\left((2p+1)\frac{\|\boldsymbol\mu\|}{\sigma}\right)\right].\label{eq:lda_err_clean}
\end{equation}
Arguing as above, for the classifier $h_S^\text{LDA}$ that does not have access to the true $\boldsymbol\mu$, and instead estimates it from the sample $S\sim\mathcal{D}^n$, we can get
\begin{equation}
    \e_\mathcal{D}[h_S^\text{LDA}]=\frac12\left[\Phi\left(-(2p-1)\frac{\|\hat{\boldsymbol\mu}\|}{\sigma}\right)\right.\\
    \left.+p\Phi\left(-(3-2p)\frac{\|\hat{\boldsymbol\mu}\|}{\sigma}\right)+(1-p)\Phi\left((2p+1)\frac{\|\hat{\boldsymbol\mu}\|}{\sigma}\right)\right].\label{eq:lda_err_est}
\end{equation}
To demonstrate the effect of replacing $\boldsymbol\mu$ by $\hat{\boldsymbol\mu}$ in \eqref{eq:lda_err_clean}, consider the absolute difference between the last c.d.f.'s in \eqref{eq:lda_err_clean} and \eqref{eq:lda_err_est}, which we can rewrite as
\begin{multline}
    \left|\Phi\left((2p+1)\frac{\|\hat{\boldsymbol\mu}\|}{\sigma}\right)-\Phi\left((2p+1)\frac{\|\boldsymbol\mu\|}{\sigma}\right)\right|\\
    =\underbrace{\left|\Phi\left((2p+1)\frac{\|\hat{\boldsymbol\mu}\|}{\sigma}\right)-\Phi\left((2p+1)\frac{\|\boldsymbol\mu\|}{\sigma}\right)\right|}_{\le1}\mathbb{I}[\|\hat{\boldsymbol\mu}-\boldsymbol\mu\|> t]+\left|\Phi\left((2p+1)\frac{\|\hat{\boldsymbol\mu}\|}{\sigma}\right)-\Phi\left((2p+1)\frac{\|\boldsymbol\mu\|}{\sigma}\right)\right|\mathbb{I}[\|\hat{\boldsymbol\mu}-\boldsymbol\mu\|\le t].\label{eq:lda_3rd_diff}
\end{multline}
By the mean value theorem, we have
\begin{equation}
    \Phi\left((2p+1)\frac{\|\hat{\boldsymbol\mu}\|}{\sigma}\right)-\Phi\left((2p+1)\frac{\|\boldsymbol\mu\|}{\sigma}\right)=\phi\left((2p+1)\frac{\|\boldsymbol\xi\|}{\sigma}\right)\frac{(2p+1)\boldsymbol\xi^\top}{\sigma\|\boldsymbol\xi\|}(\hat{\boldsymbol\mu}-\boldsymbol\mu),\label{eq:lda_mvt}
\end{equation}
where $\boldsymbol\xi$ is a point on the segment connecting $\boldsymbol\mu$ and $\hat{\boldsymbol\mu}$. Since $\boldsymbol\xi/\|\boldsymbol\xi\|$ is a unit vector, from  \eqref{eq:lda_3rd_diff} and \eqref{eq:lda_mvt} we get
\begin{equation}
    \E_{S\sim\mathcal{D}^n}\left|\Phi\left((2p+1)\frac{\|\hat{\boldsymbol\mu}\|}{\sigma}\right)-\Phi\left((2p+1)\frac{\|\boldsymbol\mu\|}{\sigma}\right)\right|\le\Pr_{S\sim\mathcal{D}^n}[\|\hat{\boldsymbol\mu}-\boldsymbol\mu\|>t]+\frac{2p+1}{\sqrt{2\pi}\sigma}\cdot t.\label{eq:mean_abs_diff}
\end{equation}
From Lemma~\ref{lem:mom_conc}, we have
\begin{equation}
    \Pr_{S\sim\mathcal{D}^n}[\|\hat{\boldsymbol\mu}-\boldsymbol\mu\|>t]\le e^{-\Omega\left(\frac{n t^2}{d}\right)}.\label{eq:conc_ineq}
\end{equation}
Now taking $t=\sqrt{\frac{d\ln n}n}$, from \eqref{eq:mean_abs_diff} and \eqref{eq:conc_ineq}, we obtain
\begin{equation}
    \E_{S\sim\mathcal{D}^n}\left|\Phi\left((2p+1)\frac{\|\hat{\boldsymbol\mu}\|}{\sigma}\right)-\Phi\left((2p+1)\frac{\|\boldsymbol\mu\|}{\sigma}\right)\right|=\widetilde{O}\left(\sqrt{\frac{d}{n}}\right).\label{eq:diff1}
\end{equation}
In a similar fashion, we can derive bounds for the other two differences:
\begin{align}
    \E_{S\sim\mathcal{D}^n}\left|\Phi\left(-(2p-1)\frac{\|\hat{\boldsymbol\mu}\|}{\sigma}\right)-\Phi\left(-(2p-1)\frac{\|{\boldsymbol\mu}\|}{\sigma}\right)\right|&=\widetilde{O}\left(\sqrt{\frac{d}{n}}\right),\label{eq:diff2}\\
    \E_{S\sim\mathcal{D}^n}\left|\Phi\left(-(3-2p)\frac{\|\hat{\boldsymbol\mu}\|}{\sigma}\right)-\Phi\left(-(3-2p)\frac{\|{\boldsymbol\mu}\|}{\sigma}\right)\right|&=\widetilde{O}\left(\sqrt{\frac{d}{n}}\right).\label{eq:diff3}
\end{align}
Finally, from \eqref{eq:lda_err_clean}, \eqref{eq:lda_err_est}, \eqref{eq:diff1}, \eqref{eq:diff2}, and \eqref{eq:diff3}, we have
\begin{align*}
    \E_{S\sim\mathcal{D}^n}\left[\e_{\mathcal{D}}[h_S^\text{LDA}]\right]&=\E_{S\sim\mathcal{D}^n}\left[\e_\mathcal{D}[h^\text{LDA}]\right]+\E_{S\sim\mathcal{D}^n}\left[\e_{\mathcal{D}}[h_S^\text{LDA}]-\e_\mathcal{D}[h^\text{LDA}]\right]\\
    &=\frac12\left[\Phi\left(-(2p-1)\frac{\|\boldsymbol\mu\|}{\sigma}\right)
    +p\Phi\left(-(3-2p)\frac{\|\boldsymbol\mu\|}{\sigma}\right)+(1-p)\Phi\left((2p+1)\frac{\|\boldsymbol\mu\|}{\sigma}\right)\right]+\widetilde{O}\left(\sqrt{\frac{d}{n}}\right),
\end{align*}
which concludes the proof.

\subsection{Proof of Lemma~\ref{lem:mda_err}}\label{sec:mda_err_proof}

Recall that the MDA classifier that has access to the true values of the parameters in data generation model $\mathcal{D}$ (defined by \eqref{eq:prior}--\eqref{eq:neg_c2}) can be written as
$$
h^\text{MDA}(\mathbf{x})=\begin{cases}
+1\quad&\text{if}\quad\frac12f(\mathbf{x};\boldsymbol\mu,\sigma^2\mathbf{I})\ge\frac{p}{2}f(\mathbf{x};-\boldsymbol\mu,\sigma^2\mathbf{I})\quad\text{and}\quad\frac12f(\mathbf{x};\boldsymbol\mu,\sigma^2\mathbf{I})\ge\frac{1-p}{2}f(\mathbf{x};3\boldsymbol\mu,\sigma^2\mathbf{I}),\\
-1 &\text{otherwise}.
\end{cases}
$$
where $f(\mathbf{x};\boldsymbol\mu,\boldsymbol\Sigma)$ is given by \eqref{eq:gauss_pdf}. Its error is decomposed as in the case of LDA, i.e.
\begin{align}
\e_\mathcal{D}[h^\text{MDA}]=\Pr_{\mathbf{X},Y\sim\mathcal{D}}[h^\text{MDA}(\mathbf{X})\ne Y]&=\Pr[h^\text{MDA}(\mathbf{X})\ne Y\mid Y=+1]\cdot\Pr[Y=+1]\notag\\
&+\Pr[h^\text{MDA}(\mathbf{X})\ne Y\mid Y=-1,K=1]\cdot\Pr[Y=-1,K=1]\notag\\
&+\Pr[h^\text{MDA}(\mathbf{X})\ne Y\mid Y=-1,K=2]\cdot\Pr[Y=-1,K=2].\label{eq:mda_err_decomp}
\end{align}
We upper-bound the three conditional probabilities separately. For the positive class, we have
\begin{align}
&\Pr[h^\text{MDA}(\mathbf{X})\ne Y\mid Y=+1]=\Pr[h^\text{MDA}(\mathbf{X})\ne+1\mid Y=+1]\notag\\
&=\Pr\left[\left\{\frac12f(\mathbf{X};\boldsymbol\mu,\sigma^2\mathbf{I})<\frac{p}{2}f(\mathbf{X};-\boldsymbol\mu,\sigma^2\mathbf{I})\right\}\,\bigcup\,\left\{\frac12f(\mathbf{X};\boldsymbol\mu,\sigma^2\mathbf{I})<\frac{1-p}{2}f(\mathbf{X};3\boldsymbol\mu,\sigma^2\mathbf{I})\right\}\bigm| Y=+1\right]\notag\\
&\le\Pr\left[-\frac{\|\mathbf{X}-\boldsymbol\mu\|^2}{2\sigma^2}<\ln p-\frac{\|\mathbf{X}+\boldsymbol\mu\|^2}{2\sigma^2}\bigm| Y=+1\right]+\Pr\left[-\frac{\|\mathbf{X}-\boldsymbol\mu\|^2}{2\sigma^2}<\ln(1-p)-\frac{\|\mathbf{X}-3\boldsymbol\mu\|^2}{2\sigma^2}\bigm| Y=+1\right]\notag\\
&=\Pr\left[\frac{1}{2\sigma^2}(2\mathbf{X})^\top(2\boldsymbol\mu)<\ln p\bigm| Y=+1\right]+\Pr\left[\frac{1}{2\sigma^2}(2\mathbf{X}-4\boldsymbol\mu)^\top(-2\boldsymbol\mu)<\ln(1-p)\bigm| Y=+1\right]\notag\\
&=\Pr\left[\mathbf{X}^\top\boldsymbol\mu<\frac{\sigma^2\ln p}{2}\bigm| Y=+1\right]+\Pr\left[(\mathbf{X}-2\boldsymbol\mu)^\top\boldsymbol\mu>-\frac{\sigma^2\ln(1-p)}{2}\bigm| Y=+1\right]\notag\\
&=\Pr_{\mathbf{Z}\sim\mathcal{N}(\mathbf{0},\mathbf{I})}\left[(\boldsymbol\mu+\sigma\mathbf{Z})^\top\boldsymbol\mu<\frac{\sigma^2\ln p}{2}\right]+\Pr_{\mathbf{Z}\sim\mathcal{N}(\mathbf{0},\mathbf{I})}\left[(\sigma\mathbf{Z}-\boldsymbol\mu)^\top\boldsymbol\mu>-\frac{\sigma^2\ln(1-p)}{2}\right]\notag\\
&=\Pr_{\mathbf{Z}\sim\mathcal{N}(\mathbf{0},\mathbf{I})}\left[\sigma\mathbf{Z}^\top\boldsymbol\mu<\frac{\sigma^2\ln p}{2}-\|\boldsymbol\mu\|^2\right]+\Pr_{\mathbf{Z}\sim\mathcal{N}(\mathbf{0},\mathbf{I})}\left[\sigma\mathbf{Z}^\top\boldsymbol\mu>\|\boldsymbol\mu\|^2-\frac{\sigma^2\ln(1-p)}{2}\right]\notag\\
&=\Pr_{Z\sim\mathcal{N}(0,1)}\left[Z<\frac{\sigma\ln p}{2\|\boldsymbol\mu\|}-\frac{\|\boldsymbol\mu\|}{\sigma}\right]+\Pr_{Z\sim\mathcal{N}(0,1)}\left[Z>\frac{\|\boldsymbol\mu\|}{\sigma}-\frac{\sigma\ln(1-p)}{2\|\boldsymbol\mu\|}\right]\notag\\
&=\Phi\left(-\frac{\|\boldsymbol\mu\|}\sigma+\frac{\sigma\ln p}{2\|\boldsymbol\mu\|}\right)+\Phi\left(-\frac{\|\boldsymbol\mu\|}{\sigma}+\frac{\sigma\ln(1-p)}{2\|\boldsymbol\mu\|}\right).\label{eq:mda_err_pos}
\end{align}
For the majority subpopulation of the negative class,
\begin{align}
&\Pr[h^\text{MDA}(\mathbf{X})\ne Y\mid Y=-1,K=1]=\Pr[h^\text{MDA}(\mathbf{X})\ne-1\mid Y=-1,K=1]\notag\\
&=\Pr\left[\left\{\frac12f(\mathbf{X};\boldsymbol\mu,\sigma^2\mathbf{I})\ge\frac{p}{2}f(\mathbf{X};-\boldsymbol\mu,\sigma^2\mathbf{I})\right\}\,\bigcap\,\left\{\frac12f(\mathbf{X};\boldsymbol\mu,\sigma^2\mathbf{I})\ge\frac{1-p}{2}f(\mathbf{X};3\boldsymbol\mu,\sigma^2\mathbf{I})\right\}\bigm| Y=-1,K=1\right]\notag\\
&\le\Pr\left[\mathbf{X}^\top\boldsymbol\mu\ge\frac{\sigma^2\ln p}{2}\bigm| Y=-1,K=1\right]=\Pr_{\mathbf{Z}\sim\mathcal{N}(\mathbf{0},\mathbf{I})}\left[\sigma\mathbf{Z}^\top\boldsymbol\mu\ge\|\boldsymbol\mu\|^2+\frac{\sigma^2\ln p}{2}\bigm| Y=-1,K=1\right]\notag\\
&=\Pr_{Z\sim\mathcal{N}(0,1)}\left[Z\ge\frac{\|\boldsymbol\mu\|}{\sigma}+\frac{\sigma\ln p}{2\|\boldsymbol\mu\|}\right]=\Phi\left(-\frac{\|\boldsymbol\mu\|}{\sigma}-\frac{\sigma\ln p}{2\|\boldsymbol\mu\|}\right).\label{eq:mda_err_neg1}
\end{align}
And for the minority subpopulation of the negative class,
\begin{align}
&\Pr[h^\text{MDA}(\mathbf{X})\ne Y\mid Y=-1,K=2]=\Pr[h^\text{MDA}(\mathbf{X})\ne-1\mid Y=-1,K=2]\notag\\
&=\Pr\left[\left\{\frac12f(\mathbf{X};\boldsymbol\mu,\sigma^2\mathbf{I})\ge\frac{p}{2}\phi(\mathbf{X};-\boldsymbol\mu,\sigma^2\mathbf{I})\right\}\,\bigcap\,\left\{\frac12\phi(\mathbf{X};\boldsymbol\mu,\sigma^2\mathbf{I})\ge\frac{1-p}{2}\phi(\mathbf{X};3\boldsymbol\mu,\sigma^2\mathbf{I})\right\}\bigm| Y=-1,K=2\right]\notag\\
&\le\Pr\left[(2\mathbf{X}-4\boldsymbol\mu)^\top\boldsymbol\mu\le-\sigma^2\ln(1-p)\bigm| Y=-1,K=2\right]=\Pr_{\mathbf{Z}\sim\mathcal{N}(\mathbf{0},\mathbf{I})}\left[(2\sigma\mathbf{Z}+2\boldsymbol\mu)^\top\boldsymbol\mu\le-\sigma^2\ln(1-p)\right]\notag\\
&=\Pr_{\mathbf{Z}\sim\mathcal{N}(\mathbf{0},\mathbf{I})}\left[2\sigma\mathbf{Z}^\top\boldsymbol\mu\le-\|\boldsymbol\mu\|^2-\sigma^2\ln(1-p)\right]=\Pr_{Z\sim\mathcal{N}(0,1)}\left[Z\le-\frac{\|\boldsymbol\mu\|}{\sigma}-\frac{\sigma\ln(1-p)}{\|\boldsymbol\mu\|}\right]=\Phi\left(-\frac{\|\boldsymbol\mu\|}{\sigma}-\frac{\sigma\ln(1-p)}{\|\boldsymbol\mu\|}\right).\label{eq:mda_err_neg2}
\end{align}
Now \eqref{eq:mda_err_decomp}, \eqref{eq:mda_err_pos}, \eqref{eq:mda_err_neg1}, and \eqref{eq:mda_err_neg2} imply the error bound
\begin{multline}
\e_{\mathcal{D}}[h^\text{MDA}]\\
\le\frac12\left[\Phi\left(-\frac{\|\boldsymbol\mu\|}\sigma+\frac{\sigma\ln p}{2\|\boldsymbol\mu\|}\right)+\Phi\left(-\frac{\|\boldsymbol\mu\|}{\sigma}+\frac{\sigma\ln(1-p)}{2\|\boldsymbol\mu\|}\right)+p\cdot \Phi\left(-\frac{\|\boldsymbol\mu\|}{\sigma}-\frac{\sigma\ln p}{2\|\boldsymbol\mu\|}\right)+(1-p)\cdot \Phi\left(-\frac{\|\boldsymbol\mu\|}\sigma-\frac{\sigma\ln(1-p)}{2\|\boldsymbol\mu\|}\right)\right].\label{eq:mda_err_clean}    
\end{multline}
As in the proof of Lemma~\ref{lem:lda_err}, estimating $\boldsymbol\mu$ by the method of moments adds the term $\widetilde{O}\left(\sqrt{d/n}\right)$, hence the expected error of the MDA classifier is
\begin{multline*}
\E_{S\sim\mathcal{D}^n}\left[\e_{\mathcal{D}}[h^\text{MDA}_S]\right]\\
\le\frac12\left[\Phi\left(-\frac{\|\boldsymbol\mu\|}\sigma+\frac{\sigma\ln p}{2\|\boldsymbol\mu\|}\right)+\Phi\left(-\frac{\|\boldsymbol\mu\|}{\sigma}+\frac{\sigma\ln(1-p)}{2\|\boldsymbol\mu\|}\right)+p\cdot \Phi\left(-\frac{\|\boldsymbol\mu\|}{\sigma}-\frac{\sigma\ln p}{2\|\boldsymbol\mu\|}\right)+(1-p)\cdot \Phi\left(-\frac{\|\boldsymbol\mu\|}\sigma-\frac{\sigma\ln(1-p)}{2\|\boldsymbol\mu\|}\right)\right]\\
+\widetilde{O}\left(\sqrt{\frac{d}{n}}\right).
\end{multline*}

\subsection{Proof of Theorem~\ref{thm:main}}\label{sec:main_proof}
For convenience, we denote $\nu=\frac{\|\boldsymbol\mu\|}\sigma$ and $p=1-\frac1t$, and rewrite the error bounds \eqref{eq:lda_err_clean} and \eqref{eq:mda_err_clean} in terms of $\nu$ and $t$:
\begin{align*}
    2\e_\mathcal{D}[h^\text{LDA}]
    %&=\Phi\left(-\nu p+\frac{l}2(1-p)\right)+p\Phi\left(-\nu(2-p)-\frac{l}2(1-p)\right)+(1-p)\Phi\left(\nu p+\frac{l}2(1-p)\right)\\
    &=\Phi\Bigl(-(2p-1)\nu\Bigr)+p\Phi\Bigl(-(3-2p)\nu\Bigr)+(1-p)\Phi\Bigl((2p+1)\nu\Bigr)\\
    &=\underbrace{\Phi\left(-\left[1-\frac2t\right]\nu\right)}_{\circled{1}}+\underbrace{\left(1-\frac1t\right)\Phi\left(-\left[1+\frac2t\right]\nu\right)}_{\circled{2}}+\underbrace{\frac1t\Phi\left(\left[3-\frac2t\right]\nu\right)}_{\circled{3}},\\
    %&=\Phi\left(-\left[1-\frac2t\right]\nu\right)+\Phi\left(-\left[1+\frac2t\right]\nu\right)+\frac1t\left\{\Phi(\nu)-\Phi\left(-\left[1+\frac2t\right]\nu\right)\right\}\\    &\ge2\Phi\left(-\nu\right)+\frac1t\left\{\Phi(\nu)-\Phi(-\nu)+\Phi(-\nu)-\Phi\left(-\nu-\frac{2\nu}{t}\right)\right\}\\
    %&=2\Phi(-\nu)+\frac1t\left\{1-2\Phi(-\nu)+\phi(\xi)\frac{2\nu}{t}\right\}\ge2\Phi(-\nu)\left(1-\frac1t\right)+\frac1t+\phi\left(\nu+\frac{2\nu}{t}\right)\frac{2\nu}{t^2}
%\end{align*}
%\begin{align*}
    2\e_\mathcal{D}[h^\text{MDA}]
    %&=\Phi\left(-\nu+\frac{\ln p}{2\nu}\right)+\Phi\left(-\frac{l}2+\frac{\ln(1-p)}{l}\right)+p\Phi\left(-\nu-\frac{\ln p}{2\nu}\right)+(1-p)\Phi\left(-\frac{l}2-\frac{\ln(1-p)}{l}\right)\\
    &=\Phi\left(-\nu+\frac{\ln p}{2\nu}\right)+\Phi\left(-\nu+\frac{\ln(1-p)}{2\nu}\right)+p\Phi\left(-\nu-\frac{\ln p}{2\nu}\right)+(1-p)\Phi\left(-\nu-\frac{\ln(1-p)}{2\nu}\right)\\
    &=\underbrace{\Phi\left(-\nu+\frac{\ln(1-1/t)}{2\nu}\right)}_{\circled{4}}+\underbrace{\Phi\left(-\nu-\frac{\ln t}{2\nu}\right)}_{\circled{5}}+\underbrace{\left(1-\frac1t\right)\Phi\left(-\nu-\frac{\ln(1-1/t)}{2\nu}\right)}_{\circled{6}}+\underbrace{\frac1t\Phi\left(-\nu+\frac{\ln t}{2\nu}\right)}_{\circled{7}}.
\end{align*}
Then
$$
    \circled{1}-\circled{4}=\Phi\left(-\nu+\frac{2\nu}t\right)-\Phi\left(-\nu+\frac{\ln(1-1/t)}{2\nu}\right)=\phi(\xi_1)\left(\frac{2\nu}{t}-\frac{\ln(1-1/t)}{2\nu}\right),
$$
where $\xi_1\in\left[-\nu+\frac{\ln(1-1/t)}{2\nu},-\nu+\frac{2\nu}{t}\right]$, and therefore 
\begin{align*}
\phi(\xi_1)\ge\phi\left(-\nu+\frac{\ln(1-1/t)}{2\nu}\right)&=\frac{1}{\sqrt{2\pi}}\exp\left(-\frac12\left[\nu^2-\ln(1-1/t)+\frac{\ln^2(1-1/t)}{4\nu^2}\right]\right)\\
&=\frac{\sqrt{1-1/t}}{\sqrt{2\pi}}\exp\left(-\frac12\left[\nu^2+\frac{\ln^2(1-1/t)}{4\nu^2}\right]\right).
\end{align*}
Thus we have
$$
\circled{1}-\circled{4}\ge\frac{\sqrt{1-1/t}}{\sqrt{2\pi}}\exp\left(-\frac12\left[\nu^2+\frac{\ln^2(1-1/t)}{4\nu^2}\right]\right)\left(\frac{2\nu}{t}-\frac{\ln(1-1/t)}{2\nu}\right).
$$

Next,
\begin{align*}
    \circled{2}-\circled{6}&=\left(1-\frac1t\right)\left[\Phi\left(-\nu-\frac{2\nu}{t}\right)-\Phi\left(-\nu-\frac{\ln(1-1/t)}{2\nu}\right)\right]\\
    &=\left(1-\frac1t\right)\phi(\xi_2)\left(-\frac{2\nu}{t}+\frac{\ln(1-1/t)}{2\nu}\right).
\end{align*}
where $\xi_2\in\left[-\nu-\frac{2\nu}{t},-\nu-\frac{\ln(1-1/t)}{2\nu}\right]$, and therefore
\begin{align*}
    \phi(\xi_2)\le\phi\left(-\nu-\frac{\ln(1-1/t)}{2\nu}\right)&=\frac{1}{\sqrt{2\pi}}\exp\left(-\frac12\left[\nu^2+\ln(1-1/t)+\frac{\ln^2(1-1/t)}{4\nu^2}\right]\right)\\
&=\frac{1}{\sqrt{2\pi}\sqrt{1-1/t}}\exp\left(-\frac12\left[\nu^2+\frac{\ln^2(1-1/t)}{4\nu^2}\right]\right).
\end{align*}
Thus
$$
\circled{2}-\circled{6}\ge-\frac{\sqrt{1-1/t}}{\sqrt{2\pi}}\exp\left(-\frac12\left[\nu^2+\frac{\ln^2(1-1/t)}{4\nu^2}\right]\right)\left(\frac{2\nu}{t}-\frac{\ln(1-1/t)}{2\nu}\right).
$$
And, therefore
$$
\circled{1}-\circled{4}+\circled{2}-\circled{6}\ge0.
$$

For the term $\circled{5}$ we have
\begin{align*}
    \Phi(-\nu)-\circled{5}&=\Phi(-\nu)-\Phi\left(-\nu-\frac{\ln t}{2\nu}\right)=\phi(\xi_3)\left(\frac{\ln t}{2\nu}\right)\\
    &\ge\phi\left(-\nu-\frac{\ln t}{2\nu}\right)\frac{\ln t}{2\nu}=\frac{1}{\sqrt{2\pi}}\exp\left(-\frac12\left[\nu^2+\ln t+\frac{\ln^2 t}{4\nu^2}\right]\right)\frac{\ln t}{2\nu}\\
    &=\frac{1}{\sqrt{2\pi t}}\exp\left(-\frac12\left[\nu^2+\frac{\ln^2 t}{4\nu^2}\right]\right)\frac{\ln t}{2\nu},
\end{align*}
where $\xi_3\in\left[-\nu-\frac{\ln t}{2\nu},-\nu\right]$.
For the term $\circled{7}$ we have
\begin{align*}
    \frac1t\Phi(-\nu)-\circled{7}&=\frac1t\Phi(-\nu)-\frac1t\Phi\left(-\nu+\frac{\ln t}{2\nu}\right)=\frac1t\phi(\xi_4)\left(-\frac{\ln t}{2\nu}\right)\\
    &\ge\frac1t\phi\left(-\nu+\frac{\ln t}{2\nu}\right)\left(-\frac{\ln t}{2\nu}\right)=\frac1{t\sqrt{2\pi}}\exp\left(-\frac12\left[\nu^2-\ln t+
    \frac{\ln^2 t}{4\nu^2}\right]\right)\left(-\frac{\ln t}{2\nu}\right)\\
    &=-\frac{1}{\sqrt{2\pi t}}\exp\left(-\frac12\left[\nu^2+\frac{\ln^2 t}{4\nu^2}\right]\right)\frac{\ln t}{2\nu}.
\end{align*}
Thus
$$
-\circled{5}-\circled{7}=\Phi(-\nu)-\circled{5}+\frac1t\Phi(-\nu)-\circled{7}-\Phi(-\nu)-\frac1t\Phi(-\nu)\ge-\Phi(-\nu)-\frac1t\Phi(-\nu).
$$
Overall
\begin{align*}
2\e_\mathcal{D}[h^\text{LDA}]-2\e_\mathcal{D}[h^\text{MDA}]&\ge\left(\circled{1}-\circled{4}+\circled{2}-\circled{6}\right)+\left(\circled{3}-\circled{5}-\circled{7}\right)\\
&\ge\frac1t\Phi(\nu)-\frac1t\Phi(-\nu)-\Phi(-\nu)\\
&=\frac1t-\frac2t\Phi(-\nu)-\Phi(-\nu)=\frac1t-\left(\frac2t+1\right)\Phi(-\nu)\\
&=(1-p)-(3-2p)\Phi(-\nu) \ge(1-p)-2\Phi(-\nu),   
\end{align*}
where we used $3-2p\le2$ which is true for $p\ge\frac12$. Going back to $\boldsymbol\mu$ and $\sigma$, using the exponential bound $\Phi(-x)\le\frac12\exp(-{x^2}/{2})$ and estimating $\boldsymbol\mu$ with the method of moments gives
$$
\E_{S\sim\mathcal{D}^n}\left[\e_\mathcal{D}[h^\text{LDA}_S]-\e_\mathcal{D}[h^\text{MDA}_S]\right]\ge\frac{1-p}{2}-\exp\left(-\frac{\|\boldsymbol\mu\|^2}{2\sigma^2}\right)+\widetilde{O}\left(\sqrt{\frac{d}{n}}\right),
$$
which concludes the proof.

\subsection{Proof of Theorem~\ref{thm:train_test_diff}}\label{sec:train_test_diff_proof}
We denote the distribution given by formulas \eqref{eq:prior}--\eqref{eq:neg_c2} as $\mathcal{D}_p$. We are interested in the expected error of the classifier, which was trained on samples from $\mathcal{D}_q$, but is tested on a sample from $D_p$, i.e. $\E_{S\sim\mathcal{D}_q}\left[\e_{\mathcal{D}_p}[h_S]\right]$. Arguing in the same way as in Sections \ref{sec:lda_err_proof} and \ref{sec:mda_err_proof}, we can prove the following results for LDA and MDA classifiers.

\begin{align}
    &2\E_{S\sim\mathcal{D}^n_{1-1/t}}\left[\e_{\mathcal{D}_p}[h^\text{LDA}_S]\right]=\Phi\left(-\left[1-\frac2t\right]\nu\right)+p\Phi\left(-\left[1+\frac2t\right]\nu\right)+(1-p)\Phi\left(\left[3-\frac2t\right]\nu\right)\label{eq:lda_pq_bound_full},\\
    &2\E_{S\sim\mathcal{D}^n_{1-1/t}}\left[\e_{\mathcal{D}_p}[h^\text{MDA}_S]\right]=\Phi\left(-\nu+\frac{\ln(1-1/t)}{2\nu}\right)+\Phi\left(-\nu-\frac{\ln t}{2\nu}\right)+p\Phi\left(-\nu-\frac{\ln(1-1/t)}{2\nu}\right)+(1-p)\Phi\left(-\nu+\frac{\ln t}{2\nu}\right).\label{eq:mda_pq_bound_full}
\end{align}
The advertised bounds follow from local linear approximations of the terms involved in \eqref{eq:lda_pq_bound_full} and \eqref{eq:mda_pq_bound_full}. For example, for the first term in \eqref{eq:lda_pq_bound_full}, we have
$$
\Phi\left(-\nu+\frac{2\nu}{t}\right)=\Phi(-\nu)+\phi(\xi_t)\frac{2\nu}{t}=\Phi(-\nu)+\Theta\left(\frac1t\right),
$$
where $\xi_t\in\left[-\nu,-\nu+\frac{2\nu}{t}\right]$, and therefore for $t\ge3$, we have $\phi(-\nu)\le\phi(\xi_t)\le\phi(-\nu/3)$, which implies the bound above.

\end{document}